\documentclass[11pt]{article}

\usepackage[margin=1in]{geometry}
\usepackage{amsmath,amssymb,amsthm}
\usepackage{booktabs}
\usepackage{graphicx}
\usepackage{xcolor}
\usepackage[round]{natbib}
\usepackage[colorlinks=true,linkcolor=blue,citecolor=blue,urlcolor=blue]{hyperref}

\newtheorem{theorem}{Theorem}
\newtheorem{proposition}{Proposition}
\newtheorem{lemma}{Lemma}
\newtheorem{corollary}{Corollary}
\theoremstyle{remark}
\newtheorem{remark}{Remark}

\newcommand{\dmax}{d_{\max}}
\newcommand{\Ireach}{\mathcal{I}}

\title{When a Verified World Model Still Loses: Play-Adequacy vs Prediction-Accuracy in LLM-Synthesized Code World Models}

\author{Javier Aguilar Mart\'in\\ AGILabs (\href{https://javieraguilar.ai}{javieraguilar.ai})}

\date{}

\begin{document}

\maketitle

\begin{abstract}
Large language models (LLMs) can synthesize the rules of a game as executable code --- a \emph{Code World Model} (CWM) --- which a classical planner then searches over. The synthesized model is typically accepted when it reaches high \emph{transition accuracy} on sampled trajectories. We argue that this acceptance criterion is the wrong notion of adequacy for planning.

Our perfect-information existence results are on LLM-synthesized CWMs. We isolate the precise causal magnitude with a hand-instrumented agent that is budget-matched and play-equivalent to the incomplete synthesized CWM, and confirm the effect end-to-end through the actual synthesis pipeline at the same budget, with confidence intervals (the synthesized incomplete CWM passes the transition gate only when a material-at-cap terminal is absent from its sample, and then loses at play, with a play cost at least as large). Our imperfect-information results pair an LLM-synthesis pipeline validation (Kuhn poker) with hand-instrumented witnesses that isolate the belief-function failure.

We find four things. (1) An LLM-synthesized CWM can pass a sampling gate at 100\% transition accuracy and be $\geq 98\%$ state-accurate on the distribution the planner actually visits, yet lose systematically at play --- because the less than 1\% it gets wrong is exactly the pivotal dynamics; the play cost of the omitted rule, isolated with the play-equivalent instrument, is $0.091$ (95\% CIs separated; seed-clustered 95\% CI $[0.065, 0.117]$ over 20 seeds, $n=4800$). We call this the \emph{verified-vs-correct gap}. (2) The harm from \emph{sampling-gate blindness} follows a quantitative law, $\mathrm{danger} = \mathrm{play\_cost} \times (1 - \mathrm{rarity})^N$, where $\mathrm{rarity}$ is the probability a random play-through triggers the omitted rule and $N$ is the gate size (the number of sampled play-throughs the gate checks). This law isolates one channel --- the gate failing to \emph{observe} a rule; it does not model the distinct, empirical synthesis-residual channel in which the rule is present in the sample yet the LLM still fails to encode it (finding (3)). The $(1 - \mathrm{rarity})^N$ gate-miss factor is \emph{proven exact} (i.i.d.\ Bernoulli); $\mathrm{play\_cost}$ is measured empirically --- with a provable upper bound (the planner's query-hit mass of the error region), a witness-certified lower bound, and an exact, provable value ($\tfrac12$) on the Beacon witness. This law predicts when sampling verification is blind: harm is negligible while the rule is common enough for the gate to catch it, rises through a threshold as the rule becomes rare, and saturates at the full $\mathrm{play\_cost}$ once it almost always escapes.

(3) The failure is not repaired by more data. LLM CWM synthesis behaves as \emph{rule translation} rather than \emph{rule inference}: it correctly encodes rules it was given and, across the regimes we tested, did not infer the omitted rule --- under the GPT-5.x family (mini and large) and across the Section~\ref{sec:repair} data regimes (rules-given, no-rules, naive DAgger, proper DAgger, targeted on-manifold examples), independent of the quantity of on-manifold example transitions. (4) The same mechanism recurs on the \emph{inference} function of imperfect-information CWMs. We prove a coverage bound --- a size-$N$ random gate is identifying for the inference function (under a detectability hypothesis) when $N \gtrsim b^{\dmax}$ --- which explains why shallow games show no inference gap: Kuhn poker is provably covered at the deployed gate, and for Leduc poker the bound certifies the competent-relevant subset of information sets (info-sets) we sampled --- not all reachable info-sets, which the deployed gate is too small to guarantee. A companion enumeration-free bound certifies the undetected-error mass of any gate-passing inference function ($\leq \ln(1/\delta)/N$ under the gate's sampling distribution), extending the certificate to games too large to enumerate. We then hand-construct a minimal witness, Beacon (not a synthesized CWM), that escapes the coverage bound: a verified-but-wrong \texttt{infer\_states} function passes the inference gate (0/8156 mismatches) yet loses every game (0.000 vs fair baseline 0.500), with the danger law recurring on this new axis. Finally, we close this witnessed gap with the gate implied by the certificate: replacing just one of 2000 random play-throughs with a held-out play-through from the trusted reference planner changes the verdict from accept (0/8156 mismatches) to reject (4/8190), because the policy-guided check reaches the deployment-critical region with probability 1.

A policy-free bounded adversarial search supplies a second closure: it rejects the Beacon instrument after expanding only 17 states (34 belief checks), though such search can be exponential in general games.

Taken together, these results suggest that adequacy for LLM-synthesized world models used in planning should be measured on the search distribution or by play directly, not by prediction accuracy on sampled transitions, and that completing the specification is more effective than attempting repair by example.
\end{abstract}

\section{Introduction}

\subsection{The Code World Model paradigm}

A central observation in recent work on LLMs for game playing is that a small language model plus a well-specified world model plus classical search can outperform a much larger model used as a direct policy. The \emph{Code World Model} (CWM) paradigm makes this concrete: an LLM is prompted with a game's rules and some sampled trajectories and asked to synthesize a fully executable implementation of the game's transition dynamics. A classical planner --- typically Monte Carlo Tree Search (MCTS) \citep{coulom2006mcts} --- then searches over the synthesized world model, interacts with a referee (the true game), and is evaluated in an arena.

We reproduce this baseline on tic-tac-toe and Connect Four (Section~\ref{sec:known-games}): LLM-synthesized CWMs refined to transition accuracy 1.0 and paired with UCT-MCTS (Monte Carlo tree search with the UCT---upper confidence bounds applied to trees---selection rule \citep{kocsis2006bandit}) dominate the same LLM used as a direct policy by wide margins (e.g., 29--1 in Connect Four), reproducing the direction of the result in \citet{lehrach2025cwm} on known games. Synthesis is trivially cheap: total API cost across all runs in this paper is approximately \$2, with roughly \$0.001--0.005 per arena game for the LLM-as-policy baseline and synthesis a one-off cost.

\subsection{The implicit trust step}

Accepting a synthesized world model involves a gatekeeping step: the CWM is refined in a sandbox until it achieves transition accuracy 1.0 on a set of randomly sampled trajectories. This gate is computationally cheap and is the only barrier between synthesis and deployment in the planner.

Our central question is: \textbf{does passing this gate certify that the CWM is adequate for planning?}

The concern is structural --- it arises from the setup itself, not from any particular model or random draw (Figure~\ref{fig:reach}). A planner does not play randomly --- it concentrates search on states it deems strategically significant. If the random-trajectory gate and the planner's search distribution diverge systematically, there could exist a CWM that passes the gate yet is wrong precisely on the states that matter for competent play. We call this the \emph{verified-vs-correct gap}.

\begin{figure}[ht]
\centering
\includegraphics[width=0.86\textwidth]{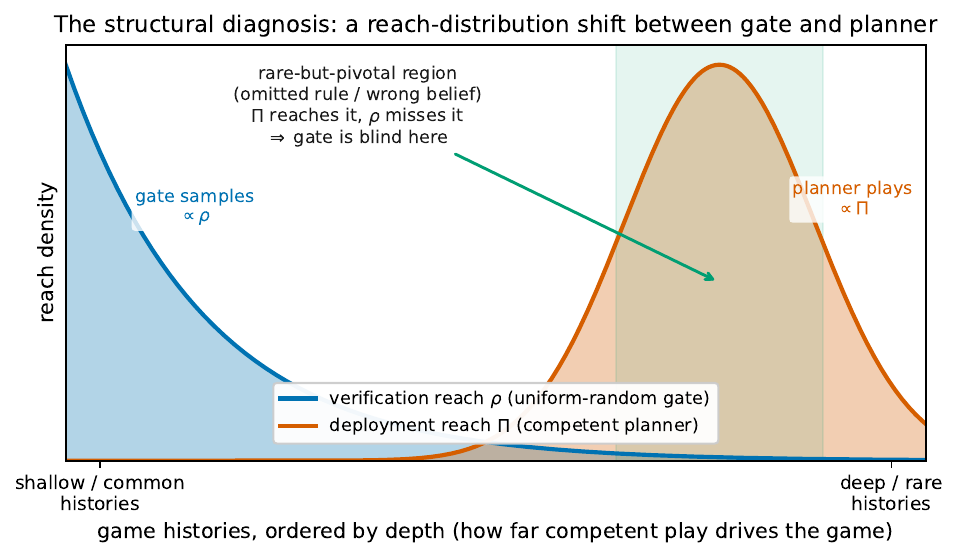}
\caption{The structural diagnosis underlying every gap in this paper, schematically. The verification policy $\rho$ (uniform-random gate) places its mass on shallow, common histories; the deployment policy $\Pi$ (competent planner) places its mass on deep, rare histories. A rare-but-pivotal region --- where an omitted transition rule (Section~\ref{sec:rare-rule}) or a wrong belief function (Section~\ref{sec:beacon}) lives --- is reached by $\Pi$ but almost never sampled by $\rho$, so the gate is structurally blind there. This is the same situation as beliefs about rarely-reached game positions being unconstrained in classical game theory (off-equilibrium-path beliefs in extensive-form games); a sampling gate has no analogous refinement to pin down that region. (Schematic, not data.)}
\label{fig:reach}
\end{figure}

\subsection{A first look: when the gate is identifying}

Our first experiments, across three game families and two knowledge regimes (rules given / rules withheld), find that the feared gap does not appear on small, fully-specified games (Section~\ref{sec:null-gap}). Whenever a CWM passes the random-trajectory gate, it is also correct on the MCTS-visited distribution; whenever it is wrong on the search distribution, it also fails the gate. For these games the random sample is \emph{identifying} --- no compact wrong hypothesis fits all training trajectories yet diverges elsewhere. This is an honest null result, and we report it as such, because understanding \emph{when} the gate is identifying is necessary to understand when it is not.

\subsection{When the gate fails: rare rules and the danger law}

The null result on small games points to the condition under which the gate can be fooled: a rule that random play almost never triggers but competent play reliably seeks out. We engineered a minimal instrument satisfying this condition (Section~\ref{sec:rare-rule}): a variant of the game army5x5a augmented with a \emph{material-at-cap} tiebreak rule whose material-terminal rarity under random play is 2.5\% (the rate at which the rule decides the game) yet which decides roughly 50\% of competent games.

A CWM that omits this rule passes the gate at transition accuracy 1.0 and is $\geq 98\%$ state-accurate on the search distribution, yet is systematically outplayed (win rate 0.404 [0.384, 0.424] vs fair baseline 0.495 [0.475, 0.515], Wilson 95\% intervals, n=4800; it loses 1.6 games for every one it wins). State accuracy is blind to the omission (\emph{dilution}: the handful of wrong states is averaged into a large pool of correct ones, so the aggregate barely moves); play is not.

We then quantify when this can happen via a law that relates harm to gate size and rule rarity (Section~\ref{sec:danger-law}), and show that the gap cannot be repaired by providing more example transitions (Section~\ref{sec:repair}).

\subsection{Extension to imperfect information}

The same mechanism appears on the \emph{inference} half of an imperfect-information CWM. Such a model must also implement an \texttt{infer\_states} function --- given a sequence of observations, reconstruct the set of possible hidden states --- and this function is gated separately. We prove (Section~\ref{sec:imperfect}) that the inference gate is identifying when the game is shallow enough, explain why poker games fall below this threshold, and construct a minimal game (Beacon) in which a verified-but-wrong \texttt{infer\_states} passes the inference gate yet loses every game. We then show experimentally that both a held-out policy-guided check and a bounded policy-free falsifier close this specific witnessed gap.

\subsection{Contributions}

This paper makes six contributions. The first three establish the perfect-information story --- the gap, the quantitative law that governs it, and why it cannot be repaired by example --- and the last three carry the same mechanism onto imperfect-information CWMs and unify the two halves of the contract:

\begin{enumerate}
\item \textbf{The verified-vs-correct gap (Section~\ref{sec:rare-rule}):} A gate-passing, $\geq 98\%$-state-accurate CWM that is systematically outplayed (win rate 0.404 [0.384, 0.424] vs fair baseline 0.495 [0.475, 0.515], Wilson 95\% intervals that do not overlap, n=4800; 1.6 losses per win). The gap arises from a rare-but-pivotal rule omitted from the specification and invisible to random-trajectory sampling. We also document the honest null --- small fully-specified games show no gap --- which clarifies the boundary conditions.

\item \textbf{A quantitative danger law (Section~\ref{sec:danger-law}):} $\mathrm{danger} = \mathrm{play\_cost} \times (1 - \mathrm{rarity})^N$. The $(1 - \mathrm{rarity})^N$ gate-miss factor is proven exact under i.i.d.\ Bernoulli sampling (Proposition~\ref{prop:gatemiss}); $\mathrm{play\_cost}$ is empirical --- but not opaquely so: we prove the upper bound $\mathrm{play\_cost} \leq \mu_{\mathrm{query}}(E)$, the planner's query-hit mass of the error region (Proposition~\ref{prop:playcostbound}), certify a lower bound by witness ($\geq 0.065$ at seed-clustered 95\%), and prove the exact value $\tfrac12$ on Beacon (Proposition~\ref{prop:beaconexact}), leaving only the exact constant at scale empirical. The law predicts a threshold below which verification is safe and above which harm saturates at the full play cost.

\item \textbf{Translation, not inference (Section~\ref{sec:repair}):} LLM CWM synthesis behaves as rule translation under the tested regimes. Across the two model sizes tested (mini, large) and every data regime we ran (naive DAgger, proper DAgger, targeted on-manifold examples), the omitted rule was not recovered from example transitions. Artificial off-manifold repair data actively corrupts synthesis.

\item \textbf{The coverage bound (Section~\ref{sec:coverage-bound}, Theorem~\ref{thm:coverage}) and Beacon (Section~\ref{sec:beacon}):} We prove that a size-$N$ random inference gate is identifying when $N \gtrsim b^{\dmax} \cdot p_{\text{chance}}^{-1} \cdot \log|\Ireach|$ --- where $b$ is the per-info-set branching factor, $\dmax$ the maximum player-action depth, $p_{\text{chance}}$ the smallest deal probability, and $|\Ireach|$ the number of reachable info-sets (all defined in Section~\ref{sec:coverage-bound}) --- which explains the absence of an inference gap in Kuhn and Leduc. We complement it with an enumeration-free certificate (Theorem~\ref{thm:errormass}): any gate-passing inference function has undetected-error hit mass $\leq \ln(1/\delta)/N$ under the gate's sampling distribution --- a bound whose constants involve no enumeration of info-sets, hence apply to games of arbitrary size, at the price of certifying bounded error mass rather than full coverage. We then construct Beacon, the minimal game that escapes the coverage bound (and realizes the reach-ratio blow-up that makes the error-mass certificate's transfer to play vacuous), obtaining 0/8156 gate mismatches alongside 0.000 play win rate. A budget-matched policy-guided gate (1999 random plus one trusted-reference play-through) rejects the same instrument with 4/8190 mismatches, and a bounded policy-free falsifier rejects after 17 expanded states (34 belief checks).

\item \textbf{The belief model is invisible to a transition gate (Section~\ref{sec:gate-blindness}):} We prove (Proposition~\ref{prop:belieftrans}) that the information partition encoded by \texttt{observation}/\texttt{infer\_states} appears in no transition tuple, so a transition-accuracy gate cannot detect a wrong belief model. We demonstrate it on masked tic-tac-toe: withholding the masking rule yields a transition-gate-perfect (1.000) but belief-wrong (\texttt{observation\_rate} 0.020, \texttt{inference\_rate} 0.180) synthesis, with a synthesized single-seed corroboration on Beacon (transition gate 1.000, \texttt{inference\_rate} 0.000). With Beacon, this gives both faces of belief-model verification --- a wrong belief can lose at play \emph{and} is invisible to a transition gate.

\item \textbf{Unification (Sections~\ref{sec:imperfect}, \ref{sec:related}):} The verified-vs-correct gap and the inference gap are one mechanism on two halves of the CWM contract. A size-$N$ random gate fails exactly on events of random-reach probability $\lesssim 1/N$ that competent play nonetheless reaches; the harm is $(\text{consequence}) \times P(\text{gate miss})$, with $P(\text{gate miss}) \approx e^{-Nr}$ for transition rules and $\approx e^{-N \cdot p_{\text{chance}} \cdot b^{-\dmax}}$ per inference info-set.
\end{enumerate}

\section{Setup and Methods}

\subsection{The CWM contract}

A Code World Model is a Python module implementing the following interface (the \emph{contract}):

\begin{itemize}
\item \texttt{initial\_state() $\to$ state} --- returns the starting game state.
\item \texttt{legal\_actions(state) $\to$ list[action]} --- returns the list of legal actions.
\item \texttt{apply\_action(state, action) $\to$ state} --- returns the next state.
\item \texttt{is\_terminal(state) $\to$ bool} --- returns True if the game is over.
\item \texttt{returns(state) $\to$ list[float]} --- returns the utility vector (only valid on terminal states).
\end{itemize}

Each state is a Python dictionary containing at minimum \texttt{board} and \texttt{current\_player} fields. For imperfect-information games (Section~\ref{sec:imperfect}) the contract is extended with:

\begin{itemize}
\item \texttt{observation(state, player) $\to$ obs} --- returns the observation available to a player at a state.
\item \texttt{initial\_states(obs, player) $\to$ list[state]} --- returns all states consistent with a first observation.
\item \texttt{infer\_states(history\_obs, player) $\to$ list[state]} --- returns all states consistent with a sequence of observations.
\end{itemize}

This contract mirrors the minimal interface required by UCT-MCTS for perfect-information games and determinized MCTS for imperfect-information games.

\subsection{Synthesis and the gate}
\label{sec:gate}

An LLM (Azure OpenAI Global Standard deployments \texttt{gpt-5.4}, \texttt{gpt-5.4-mini}, \texttt{gpt-5-nano}; snapshot \texttt{gpt-5.4-2026-03-05}) is prompted with the game's \texttt{RULES\_TEXT} and a set of random-policy trajectories. The prompt asks the model to synthesize a complete Python module implementing the contract.

The synthesized module is then refined in a sandbox: a referee evaluates the sampled trajectories through both the synthesized CWM and the ground-truth oracle, and any discrepancy produces an error message fed back to the LLM for correction. Unless otherwise stated, refinement re-uses the same trajectory sample each iteration (the synthesis-pipeline study of Section~\ref{sec:repair} is the exception --- there refinement draws a fresh batch each iteration). Refinement continues until the synthesized CWM achieves \emph{transition accuracy 1.0} on the trajectory sample, or until a refinement budget is exhausted. Passing this test we call \emph{passing the gate}.

Transition accuracy is the fraction of (state, action, next-state) tuples from randomly sampled play-throughs on which the synthesized CWM agrees with the ground truth. This is the gate's own lens; Sections~\ref{sec:gap} and~\ref{sec:danger-law} argue it is the wrong lens for play-adequacy.

For imperfect-information games, an analogous \emph{inference gate} measures accuracy on the \texttt{infer\_states} function over the observations generated by random play.

\subsection{Planning}
\label{sec:planning}

For perfect-information games, planning uses UCT-MCTS with a configurable simulation budget (200--600 per move, depending on the experiment; see individual sections). The planner operates entirely on the synthesized CWM and never queries the true game during search; only the arena referee is the true game.

For imperfect-information games, planning uses determinized MCTS: at each decision point, the planner samples a set of possible hidden-state completions from \texttt{infer\_states}, runs UCT-MCTS on each determinization, and aggregates votes \citep{cowling2012ismcts,long2010pimc}. The planner is hardened to tolerate a faulty \texttt{infer\_states} (raising exceptions, returning empty lists) via a legal-fallback mechanism, ensuring arena runs do not abort when the CWM is deliberately instrumented to be wrong.

Two alternatives deserve a word on why they are not the deployed planner. \emph{ISMCTS} \citep{cowling2012ismcts} is the natural CWM-compatible refinement of determinization-per-decision; we use the simpler variant to reuse the perfect-information harness, not because it is preferable --- a choice that is, if anything, conservative for the Beacon result (Beacon is pure disambiguation, the regime where ISMCTS most cleanly outperforms determinization, so the gap would be at least as clean under ISMCTS). \emph{CFR} \citep{zinkevich2007cfr} is the equilibrium standard; it is not the deployed planner but we do run it as an analysis tool --- \texttt{cwm.cfr} implements external-sampling MCCFR \citep{lanctot2009mccfr} and full-tree CFR+ \citep{tammelin2014cfrplus} over the same contract, used in Section~\ref{sec:coverage-bound} to measure coverage against equilibrium reach.

Throughout the paper, \textbf{competent play} denotes play under the deployed planner $\Pi$ (UCT-MCTS for perfect information, determinized MCTS for imperfect information, at the stated simulation budget). It is a heuristic proxy for skilled play, \emph{not} an equilibrium strategy: ``competent'' should be read as ``what the deployed planner does,'' never as ``optimal.'' All gap, danger, and reach statements below are made with respect to this deployed-planner distribution unless we explicitly say ``equilibrium.''

\subsection{Metrics}

We distinguish two families of metrics throughout the paper:

\textbf{The gate's lens (prediction accuracy):}
\begin{itemize}
\item \emph{Transition accuracy} --- agreement rate on (state, action, next-state) under random play.
\item \emph{State accuracy} --- fraction of states in a distribution where the CWM agrees with the ground truth on all contract outputs. Write $\mathrm{agree}(D)$ for state accuracy measured on distribution $D$; below, $D_{\text{gate}}$ is the random-trajectory (gate) distribution, $D_{\text{cwm}}$ the distribution the synthesized CWM's own planner visits, and $D_{\text{truth}}$ the distribution the ground-truth planner visits.
\item \emph{gap} (a.k.a.\ \emph{gap\_cwm}) --- $\mathrm{agree}(D_{\text{gate}}) - \mathrm{agree}(D_{\text{cwm}})$: how much more the CWM agrees with the truth where it was verified than where it actually plays. This is the quantity the verified-vs-correct gap is about.
\item \emph{gap\_truth} --- $\mathrm{agree}(D_{\text{gate}}) - \mathrm{agree}(D_{\text{truth}})$: the same difference but against the \emph{ground-truth} planner's visited distribution rather than the CWM's own.
\item \emph{gap\_max} --- the maximum of \emph{gap} across the synthesis seeds of a condition (its worst-case value, used when we report ``gap $\approx 0$, gap\_max $0.016$'').
\item \emph{observation\_rate} --- fraction of sampled (state, player) pairs on which the synthesized \texttt{observation} returns the ground-truth observation; the discriminator for a correct belief model (Section~\ref{sec:gate-blindness}).
\item \emph{inference\_rate} --- fraction of sampled observation histories on which the synthesized \texttt{infer\_states} returns exactly the ground-truth set of consistent states. For both belief-surface rates, a \emph{crashing} synthesized function is a synthesis-robustness failure, not a wrong belief: execution errors are excluded from the rate denominators and reported separately, so a rate scores only the cases that actually ran (a wrong-but-running output still counts as a mismatch).
\end{itemize}

\textbf{The decision-relevant lens (play performance):}
\begin{itemize}
\item \emph{Win rate} --- fraction of games won by the CWM+MCTS agent versus a ground-truth+MCTS opponent in an arena refereed by the true game, counting draws as $\tfrac12$ (the per-game score is $X \in \{0, \tfrac12, 1\}$), with Wilson score 95\% confidence intervals \citep{wilson1927} computed as if the score were binomial. The binomial treatment is conservative, not an approximation error: a $\{0, \tfrac12, 1\}$ score with mean $p$ has variance $p(1-p) - \tfrac14 P(\mathrm{draw}) \leq p(1-p)$, so the reported intervals are never too narrow, and every CI-separation claim below holds a fortiori.
\end{itemize}

Fairness baselines are established by running truth-vs-truth arenas (ground-truth+MCTS against itself), which should produce win rates near 0.5 for balanced games. Deviations from 0.5 in the fairness baseline indicate start-order imbalance or search-budget asymmetry; we report them alongside each play result.

\subsection{Experimental configuration}
\label{sec:config}

Table~\ref{tab:config} collects the fixed components of the pipeline so the experiments are reproducible from the paper alone; per-experiment quantities (simulation budget, gate size $N$, game) are stated in each section, and the full harness is in \texttt{docs/EXPERIMENTS.md} and the cited scripts.

\begin{table}[ht]
\centering
\small
\begin{tabular}{@{}l p{0.72\textwidth}@{}}
\toprule
Component & Setting \\
\midrule
Synthesizer LLM & Azure OpenAI \texttt{gpt-5.4} / \texttt{-mini} / \texttt{-nano}, snapshot \texttt{gpt-5.4-2026-03-05} \\
Synthesis prompt & \texttt{RULES\_TEXT} + a set of random-policy trajectories; asks for a full contract module \\
Gate sampler & uniform-random policy (\texttt{rng.choice(legal\_actions)}), seeded per run \\
Gate criterion & transition accuracy $=1.0$ on the sampled (state, action, next-state) tuples \\
Refinement & re-uses the same sample each iteration (fresh batch only in Section~\ref{sec:repair}); budget $\leq 5$ iters \\
Held-out gate & none separate --- the training trajectories \emph{are} the gate (Section~\ref{sec:danger-law}) \\
Planner (perfect info) & UCT-MCTS, exploration constant $c = 1.41 \approx \sqrt{2}$ (see note), uniform-random rollout to terminal \\
Planner (imperfect info) & determinized MCTS \citep{cowling2012ismcts,long2010pimc} with legal-fallback hardening \\
Selection tie-break & first argmax of UCT; move chosen = most-visited child (first on ties) \\
Simulation budget & 200--600 per move (stated per experiment) \\
Arena & start side alternated every game ($i \bmod 2$); distinct RNG seeds per agent and arena \\
State accuracy & agreement with ground truth on \emph{all} contract outputs at a state \\
Terminal-legal convention & \texttt{legal\_actions} on terminal states excluded (planner never queries it) \\
\bottomrule
\end{tabular}
\caption{Fixed experimental configuration. Per-experiment values (budget, gate $N$, game) are stated in context.}
\label{tab:config}
\end{table}

\paragraph{Note on the exploration constant.} We use $c = 1.41 \approx \sqrt{2}$, the canonical UCT value \citep{kocsis2006bandit}, as a fixed convention rather than a tuned hyperparameter --- search strength is set by the simulation budget, and the planner serves only as a reproducible fixed-strength instrument. The $\sqrt{2}$ constant is derived for rewards in $[0, 1]$; our \texttt{returns} lie in $\{-1, 0, +1\}$, so a strictly range-calibrated value would be $\approx 2\sqrt{2}$. We verified that this does not affect the reported quantities: across the perfect-information games, $c = 1.41$ and $c = 2\sqrt{2}$ yield identical win/draw/loss outcomes against fixed opponents (uniform-random and, on tic-tac-toe, a perfect minimax solver) at 200--600 simulations, and on Beacon --- the adversarial instrument for the coverage gap --- the planner follows an identical forced line under both constants. Trajectories on the open games diverge only by selecting among equally-optimal moves, leaving play strength and gap measurements unchanged, so no re-run with a recalibrated constant is needed.

\subsection{Games}
\label{sec:games}

We use the following games, selected to span known, novel, and partially-known regimes:

\begin{itemize}
\item \textbf{Tic-tac-toe} --- well-known; used to verify the CWM paradigm.
\item \textbf{Connect Four} --- well-known; used to verify the paradigm and as a negative control for the danger law.
\item \textbf{Generalized tic-tac-toe 6$\times$6 win-4} --- a parameterized variant ($m,n,k$ game with $m=n=6$, $k=4$); model has some prior knowledge.
\item \textbf{army5x5a} --- a generalized chess game from the DeepMind CWM paper (arXiv:2510.04542, Appendix H.5): a 5$\times$5 board with infantry, cavalry, and general pieces, win by capturing the opponent's general. The paper's public release (arXiv, 2025-10-06) post-dates the GPT-5.4 knowledge cutoff (2025-08-31), so army5x5a falls outside the training window; a declarative probe independently confirmed the model does not know the detailed movesets --- asked to state the board size, piece counts, movement offsets, win condition, and starting position or decline each explicitly, the model answered ``I do not know'' to all five parts, while a positive-control question on Kuhn poker in the same session was answered with complete, correct recall (question and full response persisted in \texttt{results/declarative\_recall\_probe.json}; \texttt{scripts/declarative\_recall\_probe.py}) --- making it a genuine translation target.
\item \textbf{army5x5a + material-at-cap} --- the above, with an added tiebreak rule: if the game reaches the ply cap (100 plies by default) with both generals alive, the player with more material wins. This is the primary instrument for the rare-rule gap.
\item \textbf{Trike} \citep{erickson2020trike} --- an abstract combinatorial game in the \emph{wrong-prior} regime: in a declarative probe that allows declining, the model declines to state the mechanics (``I do not reliably know,'' all five parts; \texttt{results/declarative\_recall\_probe.json}), yet a no-rules synthesis produces confidently wrong mechanics that fail the gate in every seed (0/5, Section~\ref{sec:null-gap}) --- the wrong prior manifests under synthesis pressure, not in declarative recall. Real rules: place a disc on an empty cell on the shared pawn's line, then move the pawn to that disc; game ends when the pawn is surrounded; score = discs adjacent to the pawn.
\item \textbf{Kuhn poker} \citep{kuhn1950poker} --- a minimal imperfect-information game; used to validate the imperfect-information pipeline.
\item \textbf{Leduc poker} \citep{southey2005leduc} --- a slightly larger poker game (6-card deck, community card, two betting rounds); used for the coverage-bound corollary.
\item \textbf{Beacon} --- a minimal game constructed specifically to instantiate the positive imperfect-information gap; described in Section~\ref{sec:imperfect}.
\end{itemize}

Non-triviality of the novel games is confirmed empirically: MCTS beats random play from both sides on generalized tic-tac-toe 6$\times$6, army5x5a, and Trike with zero losses (Section~\ref{sec:null-gap}).

\subsection{Cost}

API synthesis is trivially cheap: approximately \$0.043--\$0.135 per game family for the known-game runs; roughly \$0.81 total for the gap-grid across three games and ten seeds. The CPU bottleneck is MCTS. For the danger law (Section~\ref{sec:danger-law}) we exploit the fact that $\mathrm{play\_cost}$ is approximately constant by measuring it precisely once at scale (240 games, 600 simulations) and sweeping $\mathrm{rarity}$ cheaply (3000 random games per cap setting, no MCTS required). All results and reproduction commands are in \texttt{docs/EXPERIMENTS.md}.

\section{The Gap: Accuracy \texorpdfstring{$\neq$}{≠} Play-Adequacy}
\label{sec:gap}

\subsection{Known games reproduce the paradigm but do not stress the gate}
\label{sec:known-games}

LLM-synthesized CWMs on tic-tac-toe and Connect Four pass the transition gate in 0 refinement iterations and play at well above baseline performance. For completeness (Table~\ref{tab:knowngames}):

\begin{table}[ht]
\centering
\small
\resizebox{\textwidth}{!}{%
\begin{tabular}{llcccccr}
\toprule
Game & Synthesizer & Refine iters & Trans.\ acc. & CWM W/D/L & Base illegal & CWM illegal & Total cost \\
\midrule
Tic-tac-toe  & gpt-5.4-mini & 0 & 1.0 & 18 / 10 / 2 & 6 & 0 & \$0.043 \\
Tic-tac-toe  & gpt-5-nano   & 0 & 1.0 & 21 / 8 / 1  & 5 & 0 & \$0.043 \\
Connect Four & gpt-5.4-mini & 0 & 1.0 & 29 / 0 / 1  & 0 & 0 & \$0.135 \\
Connect Four & gpt-5-nano   & 0 & 1.0 & 30 / 0 / 0  & 2 & 0 & \$0.132 \\
\bottomrule
\end{tabular}%
}
\caption{Known-game reproduction (30 games each, seed 7; CWM agent = synthesized model + MCTS; baseline = direct LLM policy).}
\label{tab:knowngames}
\end{table}

The CWM+MCTS agents dominate the direct LLM policy, replicating the paradigm's core claim. That these models reach transition accuracy 1.0 in zero refinements is itself revealing: on well-known games, the model is almost certainly \emph{recalling} the rules rather than inferring them from trajectories. Here ``accuracy 1.0 on sampled trajectories'' likely coincides with global correctness for the right reason --- but that coincidence says nothing about whether the gate is reliable in general.

\subsection{The state-agreement gap does not appear on small complete-rules games}
\label{sec:null-gap}

To measure the gap properly, we ran a grid across three knowledge regimes and two model sizes (5 synthesis seeds each, 20 self-play games, 300 simulations, train-games 40; Table~\ref{tab:nullgap}):

\begin{table}[ht]
\centering
\small
\begin{tabular}{lllccccc}
\toprule
Game & Regime & Synth & gap mean & gap max & gate-pass & median refine iters & exec-err \\
\midrule
gen\_tictactoe & correct prior & mini & 0.000 & 0.001 & 5/5 & 0 & 0 \\
gen\_tictactoe & correct prior & nano & 0.000 & 0.000 & 5/5 & 0 & 0 \\
army5x5a       & no prior      & mini & 0.002 & 0.008 & 4/5 & 0 & 0 \\
army5x5a       & no prior      & nano & n/a   & n/a   & 0/5 & -- & 0 \\
trike          & wrong prior   & mini & 0.000 & 0.000 & 4/5 & 1 & 0 \\
trike          & wrong prior   & nano & 0.000 & 0.000 & 5/5 & 0 & 0 \\
\bottomrule
\end{tabular}
\caption{State-agreement gap across knowledge regimes and model sizes.}
\label{tab:nullgap}
\end{table}

We report this as an honest null. In every regime the outcome is binary: either the CWM is correct on every evaluated state (no observed divergence on the gate, search, and truth-search distributions), with gap $\approx 0$, or it fails the gate entirely. No CWM passes the gate yet is wrong on the MCTS-visited distribution. The same pattern holds in a \texttt{-{}-no-rules} variant (synthesis from trajectories alone, with \texttt{RULES\_TEXT} withheld): gen\_tictactoe passes in 2/5 seeds via recall, with gap 0; army5x5a and Trike fail the gate entirely (0/5).

For the one game small enough to enumerate, gate-passing certifies \emph{global} correctness on reachable states --- a \emph{proof}, not just an observation. On tic-tac-toe we synthesize a CWM, confirm it passes the random gate (accuracy 1.0), then check it against the truth over the \textbf{entire reachable state space} by breadth-first enumeration (5478 reachable states, 16167 transitions): \textbf{zero mismatches on the search-relevant relation} --- \texttt{legal\_actions} on non-terminal states, \texttt{apply\_action} on every (state, legal action), \texttt{is\_terminal}, and \texttt{returns} (\texttt{scripts/exhaustive\_verify\_tictactoe.py}). (The enumeration also surfaces 880 \texttt{legal\_actions}-on-terminal-state divergences --- the synthesized code omits the \texttt{is\_terminal} guard --- but a planner never queries \texttt{legal\_actions} on a terminal state, so these are a behaviourally-irrelevant convention artifact, excluded exactly as the gap measurement does.) This is the transition-function analogue of the inference-side \emph{coverage bound} we prove later (Section~\ref{sec:coverage-bound}): when the check covers the whole reachable relation, passing it certifies global correctness by exhaustion. For the larger games the reachable space is too large to enumerate cheaply --- Trike side-6 alone exceeds 3 million reachable states (measured), and army5x5a and generalized tic-tac-toe $6\times6$ are far larger --- so there the null is a statement about the evaluated distributions, not a proof; consistently, army5x5a (the largest) is the one game with a tiny residual (gap 0.002).

The diagnosis: for small, fully-specified games, the random-trajectory sample \emph{appears identifying on the evaluated distributions} --- no compact wrong hypothesis we observed fits 40 random trajectories yet diverges on the search distribution (proven by exhaustion for tic-tac-toe; an empirical statement on the evaluated distributions for the larger games above). The gate is not weak in these regimes; it is identifying. This makes the gap harder, not easier, to construct: we need a game where random and competent play genuinely diverge.

The null is informative in a second way: it shows that the binding constraint on small games is \emph{gate-attainability}, not gap size. nano fails army5x5a outright (0/5 gate passes) because the action encoding (a \texttt{from*25+to} integer plus a ply counter) is representationally complex; mini handles it (4/5). The knowledge regime matters less than model scale and encoding complexity --- consistent with the translation hypothesis (Section~\ref{sec:repair}).

\subsection{The rare-rule instrument: verified but wrong at play (headline result)}
\label{sec:rare-rule}

The null result on fully-specified games points to the necessary condition for a gap: a rule whose random-play incidence is near zero but whose competent-play incidence is high. We searched for such rules systematically.

\paragraph{The rarity$\leftrightarrow$consequence frontier.} We tested six rules across Connect Four and army5x5a (rarity = fraction of random games the rule decides; consequence = performance change between rule-aware and rule-blind MCTS on the true game; Table~\ref{tab:frontier}):

\begin{table}[ht]
\centering
\small
\begin{tabular}{lllc}
\toprule
Base & Rule & Rarity (random) & Consequence \\
\midrule
Connect Four & last-placer-on-full-board wins & 0\%  & none \\
Connect Four & corner 4-in-a-row is poison    & 3\%  & weak \\
Connect Four & top-centre fill wins           & 12\% & strong \\
Connect Four & vertical-3 in centre wins      & 23\% & strong \\
Connect Four & 2$\times$2 square wins         & 38\% & strong \\
army5x5a     & infantry breakthrough wins     & 75\% & strong \\
\bottomrule
\end{tabular}
\caption{The rarity$\leftrightarrow$consequence anti-correlation curve.}
\label{tab:frontier}
\end{table}

Across our six-rule probe set, the pattern is that anything a planner can force, random play also stumbles into. Connect Four admits no rule in the rare-and-consequential quadrant. A random-vs-MCTS game-length measurement (\texttt{scripts/divergence.py}) confirms the diagnosis: army5x5a stands out --- under competent self-play, 43\% of games run to the 100-ply cap versus 3\% under random play (mean length 52 vs 32 plies; the competent length distribution is strongly bimodal --- quick decisive wins or cap-length maneuvering) --- while Trike and generalized tic-tac-toe behave like Connect Four (no such divergence). The median game length is unstable under this bimodal distribution; the cap-hit rate is the robust signature, and it is the one the instrument construction relies on. A game where random and competent play visit very different parts of the state space is the necessary substrate.

\paragraph{The instrument.} We constructed a variant of army5x5a with a \emph{material-at-cap} tiebreak rule: if the game reaches the ply cap (100 plies) with both generals alive, the player with more pieces wins (rather than drawing). Two rates matter and we keep them distinct. The cap is reached in 5.2\% of random games (most are equal-material draws the rule leaves unchanged); the rule's \textbf{material-terminal rarity} --- the rate at which the material-at-cap branch produces a decisive, gate-observable result --- is $r = 0.0253$ (76 of 3000 random games; $\approx 2.5\%$). The danger law (Section~\ref{sec:danger-law}) uses this material-terminal rarity. The same rule decides approximately 50\% of competent games. Implementations: \texttt{groundtruth/gen\_chess\_material.py} with paired specifications \texttt{army5x5a\_material} (complete rules) and \texttt{army5x5a\_material\_incomplete} (base rules, omitting the material-at-cap tiebreak).

\paragraph{State accuracy is the wrong lens for play-adequacy.} A CWM that omits the material-at-cap rule passes the gate (transition accuracy 1.0), and the gap\_truth is approximately 0 across all seeds (Table~\ref{tab:gatepass}):

\begin{table}[ht]
\centering
\small
\begin{tabular}{lccp{6cm}}
\toprule
Condition (mini, 5 seeds) & gate-pass & gap\_truth & note \\
\midrule
incomplete (omits rule) & 2--3/5 & 0.000 & seeds that fail the gate do so because the rule appeared in their 40 training trajectories \\
complete (control)      & 5/5    & 0.000 & --- \\
\bottomrule
\end{tabular}
\caption{The incomplete CWM passes the gate with near-zero gap\_truth.}
\label{tab:gatepass}
\end{table}

The gate-passing-but-rule-blind seeds correspond to the event that the rule is \emph{absent} from the $N \approx 40$ training trajectories, which has probability $(1-r)^N$ (Proposition~\ref{prop:gatemiss}, Section~\ref{sec:danger-law}). Here the training sample doubles as the gate, so the proposition is instantiated at $N \approx 40$. On this event the sample is consistent with both the rule-bearing and the rule-omitting world model, so by the sample-identifiability argument (Section~\ref{sec:repair}) the rule is unidentifiable and the rule-blind CWM is admissible. When the rule \emph{does} appear in the sample --- probability $1-(1-r)^N$ --- synthesis is \emph{not} guaranteed to encode it (Section~\ref{sec:repair} shows examples do not reliably teach the rule); the observed seeds are consistent with the same Bernoulli miss mechanism. Here $N$ is the training sample rather than a separate validation gate.

The divergence region (ply-cap states with unequal material) is less than 1\% of visited states, and symmetric MCTS self-play tends toward equal material, so the states where the rule-blind CWM is wrong are barely sampled --- they are too small a fraction of visited states to move the aggregate state-accuracy number.

\paragraph{Play performance is the decision-relevant lens.} We report play in two clearly separated panels: a budget-matched, instrumented, CI'd headline (Panel A) and a budget-matched, CI'd \emph{synthesized}-pipeline replication (Panel B). We deliberately do not pool instrumented and synthesized evidence into one number.

\paragraph{Panel A --- budget-matched, instrumented, CPU-only (the headline causal claim).} Identical budget across arms ($n = 2400$ games per arm --- 20 seeds $\times$ 120 games per seed per arm, 4800 games in total --- at 600 simulations), arena refereed by the true game (army5x5a + material-at-cap), measured via \texttt{scripts/play\_cost\_ci.py} / \texttt{scripts/play\_cost\_blind3.py} (Table~\ref{tab:panelA}, visualized in Figure~\ref{fig:headline}):

\begin{table}[ht]
\centering
\small
\begin{tabular}{lc}
\toprule
Arena (true game = army5x5a + material-at-cap) & win rate [Wilson 95\%] \\
\midrule
truth-vs-truth (fair baseline)        & \textbf{0.495} [0.475, 0.515] \\
rule-blind vs truth (play cost)       & \textbf{0.404} [0.384, 0.424] \\
\bottomrule
\end{tabular}
\caption{Panel A: budget-matched, instrumented headline. The Wilson 95\% intervals do not overlap.}
\label{tab:panelA}
\end{table}

\begin{figure}[ht]
\centering
\includegraphics[width=0.62\textwidth]{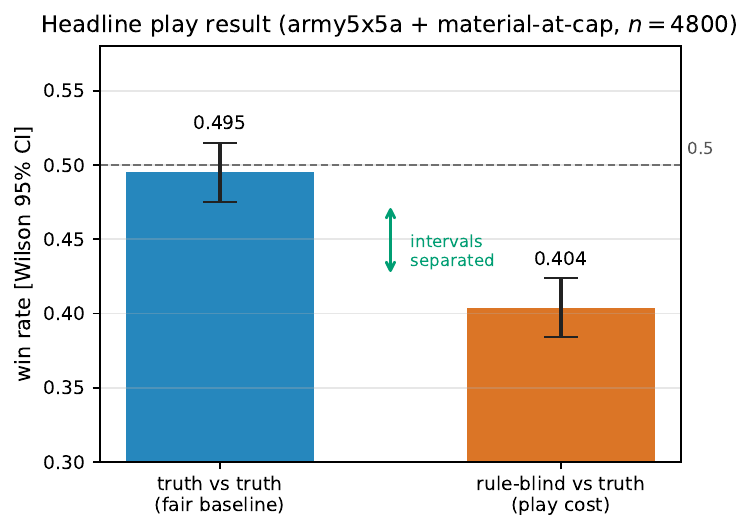}
\caption{The headline result of Table~\ref{tab:panelA}: a gate-passing, $\geq 98\%$ state-accurate, rule-blind instrument (right) loses to the fair baseline (left) at play. The Wilson 95\% intervals are visibly separated (fair lower bound $0.475 >$ rule-blind upper bound $0.424$), so the $0.091$ play cost is not a sampling artifact.}
\label{fig:headline}
\end{figure}

play\_cost = 0.091 (the fair-baseline win rate minus the rule-blind win rate). The Wilson 95\% intervals \textbf{do not overlap} (fair lower bound 0.475 $>$ rule-blind upper bound 0.424): the two intervals are \emph{CI-separated} --- they share no common value, so the gap is not a sampling artifact. We use ``CI-separated'' throughout for this non-overlap of 95\% confidence intervals.

\paragraph{Seed-clustered inference (the seed, not the game, as the unit).} The Wilson intervals above pool games and so treat each game as independent. Because games within a seed share a synthesis/instrument draw and an RNG stream, we also report a more conservative analysis that takes the \emph{seed} as the independent unit. The paired-by-seed difference (which cancels start-side and budget effects, identical across arms) has mean $0.091$ with a Student-$t$ 95\% interval of $[0.065, 0.117]$ over the twenty seeds ($\mathrm{sd} = 0.055$, $\mathrm{df} = 19$), \textbf{excluding zero} (Table~\ref{tab:perseed}). Per-seed differences range from $0.00$ (two seeds) to $0.19$, so the effect is heterogeneous across seeds --- and survives the clustering anyway. As a stability check, the estimate is robust to the seed budget: runs at 3 seeds ($n=360$: fair 0.493, blind 0.376), 5 seeds ($n=600$: 0.507, 0.376, play\_cost $0.131~[0.083, 0.179]$), and 20 seeds ($n=4800$, the numbers above) each yield a point estimate inside the preceding run's interval.

\begin{table}[ht]
\centering
\small
\begin{tabular}{lcccc}
\toprule
arm (20 seeds, 120 games/seed/arm) & mean & sd & min & max \\
\midrule
fair (truth-vs-truth)  & 0.495 & 0.035 & 0.429 & 0.542 \\
rule-blind vs truth    & 0.404 & 0.041 & 0.333 & 0.467 \\
paired difference (play cost) & \textbf{0.091} & 0.055 & 0.000 & 0.192 \\
\bottomrule
\end{tabular}
\caption{Per-seed summary underlying the seed-clustered interval (per-seed values in \texttt{results/play\_cost\_ci.json}). The paired difference has mean $0.091$, Student-$t$ 95\% CI $[0.065, 0.117]$ ($\mathrm{df}=19$), excluding zero.}
\label{tab:perseed}
\end{table}

The rule-blind agent is hand-written base army5x5a. Two distinct equivalences are at work and we keep them separate. (i) Base-vs-truth isolates the \emph{rule's} play cost \emph{by construction}: both arms are exact hand-written games that differ only in the material-at-cap rule, so the contrast is exact at any budget without invoking the LLM. (ii) The bridge to the synthesized pipeline is \emph{empirical}, not proven: base is play-equivalent to the incomplete synthesized CWM only up to the measured gap\_cwm $\approx 0$ (a tiny residual --- gap\_max $0.016$ on one seed, Section~\ref{sec:gap}), not by a proof of functional equivalence outside the rare region. This is why the headline causal claim rests on an instrumented, budget-matched, CI'd comparison rather than on the synthesized runs; Panel B then shows the \emph{synthesized} pipeline reproduces the play deficit in direction and mechanism, with its own confidence intervals, at Panel A's budget (at a larger magnitude --- see Panel B).

\paragraph{Panel B --- LLM-synthesized pipeline, at Panel A's budget, with CIs.} We ran the actual synthesis pipeline (GPT-5.4-mini) at Panel A's budget --- 20 seeds, 120 games/seed, 600 simulations, refinement capped at 5 iterations (\texttt{scripts/play\_cost\_synth\_ci.py}) --- synthesizing an incomplete-rules and a complete-rules CWM per seed against the same paired fair baseline. Unlike the hand-written instrument, a \emph{synthesized} incomplete CWM is rule-blind only when a material-at-cap terminal is absent from its gate sample; when such a terminal is present the material-region transitions are inexplicable to a program that omits the rule, so the gate cannot reach $1.0$. This is exactly what we observe, and it is the identifiability event made concrete (Table~\ref{tab:panelB}):

\begin{table}[ht]
\centering
\small
\setlength{\tabcolsep}{5pt}
\begin{tabular}{lccc}
\toprule
Synthesized CWM & gate-passing & win rate [Wilson 95\%] & play\_cost [95\% CI] \\
\midrule
incomplete-rules & 9/20 ($n=1080$) & 0.345 [0.317, 0.374] & \textbf{0.154} [0.135, 0.173] \\
complete-rules   & 20/20 ($n=2400$) & 0.471 [0.451, 0.491] & 0.024 [0.000, 0.047] \\
\bottomrule
\end{tabular}
\caption{Panel B: the synthesized pipeline replicates the play deficit end-to-end. The incomplete-rules CWM passes the gate only when no material-at-cap terminal is in its sample (all 9 gate-passing seeds; 0/10 material-terminal-present seeds reach gate 1.0, stalling at $0.83$--$0.999$), and when it does it loses --- play\_cost $0.154$, seed-clustered 95\% CI excluding zero. The complete-rules control is near parity (a small residual deficit, $0.024$, whose 95\% CI just excludes zero). Win-rate CIs are pooled Wilson over $n$ arena games; play\_cost CIs are seed-clustered.}
\label{tab:panelB}
\end{table}

The synthesized incomplete CWM, whenever it passes the gate, loses at play with $\mathrm{play\_cost} = 0.154$ ($[0.135, 0.173]$, seed-clustered over the 9 gate-passing seeds, excluding zero; every one of the 9 is positive, $0.11$--$0.19$), against the paired fair baseline over those same seeds ($0.499$). This replicates Panel A's finding in direction and mechanism --- a positive play deficit whose seed-clustered CI excludes zero --- measured end-to-end through synthesis rather than a hand-written instrument. The magnitude is \emph{larger}, not equal: Panel B's $[0.135, 0.173]$ is disjoint from Panel A's isolated rule cost ($0.091$, $[0.065, 0.117]$), as expected, since the synthesized program can carry imperfections beyond the omitted rule --- so the end-to-end harm is at least as large as the rule's isolated cost. The gate-pass structure is itself the finding: gate-passing coincides almost perfectly with material-terminal absence (9/9),\footnote{The coincidence is one-directional: every gate-passing seed is material-terminal-absent, but not every material-terminal-absent seed passes. Of the ten material-terminal-absent seeds, nine reach gate $1.0$; the tenth (seed 11) stalls at gate $0.808$ --- a base-game synthesis failure (the LLM does not reproduce even the rule-free base game to gate $1.0$), distinct from all three danger-law channels, which concern the omitted rule specifically. It is excluded from the play measurement, which conditions on gate $1.0$.} so material-terminal absence is \emph{necessary} for the gate to pass here (though not sufficient --- seed 11 fails for the synthesis reason above); the harm operates through the sampling-miss event of the danger law (Section~\ref{sec:danger-law}), and when a material-terminal \emph{is} sampled the gate rejects the rule-omitting program rather than certifying it. The complete-rules control is near parity (a small residual deficit, $\mathrm{play\_cost}\ 0.024$, whose 95\% CI just excludes zero).

\paragraph{Summary.} A world model can pass transition-accuracy verification (gate 1.0), be $\geq 98\%$ state-accurate on its own search distribution and exactly correct (gap\_truth = 0) on the truth-planner distribution, and yet lose systematically at play --- because the less than 1\% it gets wrong is exactly the pivotal tactic. Transition and state accuracy are the wrong adequacy criteria for planning; play performance is the right one.

\section{A Quantitative Law of Sampling-Verification Harm}
\label{sec:danger-law}

We now characterize \emph{when} the harm from a sampling gate is large. The key observation is that a gate of $N$ random play-throughs fails to observe a rule that fires with probability $r$ per play-through with exact probability $(1 - r)^N$. We formalize this gate-miss probability and measure the remaining empirical component.

It is worth stating up front exactly which channel of harm this law captures, since the paper documents three and only one is governed by the law. (a) The \emph{gate-miss} channel: the gate never observes the rule, with probability $(1-r)^N$ --- proven exact below. (b) The \emph{play-cost} channel: conditional on the rule being absent from the synthesized model, how much a competent opponent can exploit the omission --- measured empirically ($\mathrm{play\_cost}$). The danger law is the product of channels (a) and (b): the gate-miss probability times the play-cost. (c) A distinct \emph{synthesis-residual} channel, which the law does \emph{not} model: even when the rule \emph{does} appear in the gate sample, the LLM may still fail to encode it (Section~\ref{sec:repair} shows this is the common case under repair-by-example). The law is therefore a lower bound on total harm under an idealized synthesizer that always encodes what it observes; the synthesis-residual channel is reported separately and empirically.

\subsection{The gate-miss proposition}

\begin{proposition}[gate-miss probability]
\label{prop:gatemiss}
A sampling gate draws $N$ i.i.d.\ uniform-random play-throughs and accepts the CWM if none of them triggers the rule in question (and so reveals a discrepancy between the CWM and the true game). Since each play-through triggers the rule independently with probability $r$ (the ``rarity''), the probability the gate never observes the rule is exactly
\[
P(\mathrm{miss}) = (1 - r)^N \approx e^{-Nr}.
\]
\end{proposition}

\begin{proof}
Each play-through is a Bernoulli($r$) event (rule fires / does not fire), and the $N$ plays are i.i.d.\ (uniform-random policy is memoryless). The probability all $N$ draws are non-firing is $\prod_{i=1}^{N}(1-r) = (1-r)^N$.
\end{proof}

The approximation $e^{-Nr}$ is useful for intuition but the exact expression $(1-r)^N$ is what the table below uses.

\paragraph{What $N$ counts.} $N$ is the number of \emph{distinct, independent} i.i.d.\ uniform-random play-throughs the verification pipeline actually draws that could reveal the rule. The proposition is agnostic to which pipeline stage a draw belongs to, but only \emph{fresh} draws count --- re-using the same trajectories does not increase $N$. In our synthesis pipeline the draw is a single set of training trajectories that doubles as the gate: refinement re-checks the \emph{same} trajectories (it collects no new games and so contributes nothing to $N$), and we run no separate held-out validation gate. We verified this against the implementation. The numerical instances therefore are: in the Section~\ref{sec:gap}/Section~\ref{sec:rare-rule} gap grid $N \approx 40$ (the training trajectories \emph{are} the gate, refinement reuses them); Beacon (Section~\ref{sec:beacon}) draws a separate gate of $N = 2000$; the danger-curve sweep below reports $N \in \{20, 40, 80\}$.

\begin{corollary}[danger law]
\label{cor:dangerlaw}
Let $\kappa = \mathrm{play\_cost}$ be the expected play deficit of a planner whose CWM omits the rule, conditional on the omission surviving the gate. The expected harm from a gate of size $N$ is
\[
\mathrm{danger}(N) = \kappa \cdot (1 - r)^N.
\]
\end{corollary}

The $(1-r)^N$ factor is exact (Proposition~\ref{prop:gatemiss}); $\kappa$ is the empirically-measured, game- and planner-specific consequence magnitude.

\begin{remark}[what stays empirical]
We treat $\kappa = \mathrm{play\_cost}$ as roughly constant across rarity values. This is an empirical regularity, not something the math forces. It holds here because competent MCTS reaches the ply-cap region regardless of how the rarity knob tunes $r$, so the consequence of omitting the rule is about the same whether the rule is common or rare --- as long as it escapes the gate. The invariance therefore depends on the planner consistently reaching the rule region (a property of the game and the search budget), not on anything about the sampling model. ``Empirical'' is not all-or-nothing, however: an upper bound on $\mathrm{play\_cost}$ is provable in general, its exact value is provable on solvable witnesses, and lower bounds are certifiable by witness --- see Proposition~\ref{prop:playcostbound} and Remark~\ref{rem:playcost-provable} below.
\end{remark}

\begin{remark}[why play\_cost can be measured once: the cheap/expensive split]
The rarity-invariance of $\kappa$ is what lets us measure it just \emph{once} --- expensively, with MCTS --- and reuse that single value at every rarity, while sweeping rarity \emph{cheaply} with random games (no MCTS). We call this cost-saving the \emph{cheap/expensive split}: rarity is cheap to sweep, $\mathrm{play\_cost}$ is expensive but need only be paid once. A direct measurement shows \emph{why} the two factors separate this way (\texttt{scripts/play\_cost\_reach.py}, MCTS 300 simulations, 120 games per cap; Figure~\ref{fig:mechanism}). As the cap knob is varied, the probability that \emph{competent} play reaches the cap region stays roughly flat (0.183, 0.242, 0.275 at caps 30 / 60 / 100 --- the Wilson 95\% intervals overlap pairwise, so no trend is resolved at this sample size), whereas the probability that \emph{random} play reaches it falls sharply (0.442, 0.133, 0.067 --- a $6.6\times$ drop, CI-separated end to end). So $\mathrm{play\_cost}$ rides the knob-insensitive \emph{competent} reach --- which is exactly why it is approximately rarity-invariant --- while rarity rides the knob-dependent \emph{random} reach. (The shape is robust to sample size: at 40 games per point the same pattern appears --- competent 0.200/0.200/0.225, random 0.375/0.200/0.075 --- and tripling the sample sharpens the random-reach fall while leaving the competent-reach flatness intact.) This is a mechanistic correlate, not a proof; $\mathrm{play\_cost}$-constancy remains an empirical regularity.
\end{remark}

\begin{figure}[ht]
\centering
\includegraphics[width=0.55\textwidth]{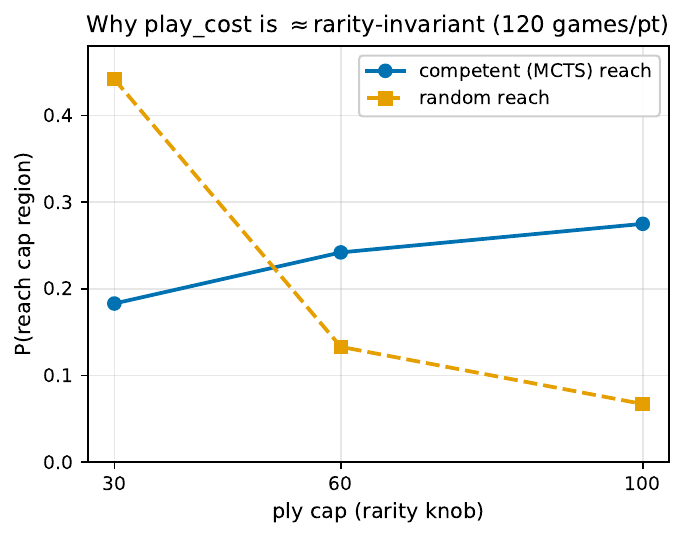}
\caption{A measured mechanistic correlate of the cheap/expensive split (\texttt{scripts/play\_cost\_reach.py}, 120 games/point, 300 simulations). Competent (MCTS) reach of the cap region is roughly flat as the cap knob varies (overlapping Wilson 95\% intervals), so $\mathrm{play\_cost}$ is approximately rarity-invariant; random reach falls steeply ($6.6\times$, CI-separated), which is what drives the rarity knob. The two factors of the danger law thus ride two different reach distributions.}
\label{fig:mechanism}
\end{figure}

\begin{remark}[reading the two factors, and which reference distribution]
\label{rem:reference-dist}
The danger law factorizes cleanly into a \emph{verification-distribution} term and a \emph{play-distribution} term: $(1-r)^N$ is the probability the verification distribution ($N$ i.i.d.\ uniform-random plays) misses the rule --- and $r$ here is a \emph{random}-play rate, measured on the gate's own uniform-random sampling rather than on the planner's distribution --- while $\kappa = \mathrm{play\_cost}$ is the consequence measured on the \emph{deployment} distribution, the reach of the deployed planner $\Pi$. The normatively correct reference for play-adequacy is arguably neither of these but \emph{equilibrium / best-response reach}, and under a true equilibrium opponent the play-cost factor could differ: an equilibrium opponent might exploit the rule-blind agent more aggressively, or the omitted rule might sit off the equilibrium path entirely. We therefore keep the hedge ``on the distribution the planner actually visits'' throughout rather than upgrading to a distribution-free claim, and scope the ``$\approx$ constant play\_cost'' claim to the deployed self-play planner; substituting equilibrium reach for MCTS reach would shift the \emph{numbers} (a rarity- and a play-cost-under-equilibrium, measured against a solver we did not run) but not the \emph{mechanism} --- a verified model can be wrong precisely where the reference distribution concentrates and the verification distribution does not. The \emph{inference}-side coverage bound, by contrast, is \emph{equilibrium-robust} (Section~\ref{sec:coverage-bound}): there the same full-support property makes the reference-distribution question a strength rather than a caveat.
\end{remark}

Although the \emph{exact} value of $\mathrm{play\_cost}$ stays empirical, more of it is provable than the factorization above suggests. We record the provable upper bound here and collect the full picture in Remark~\ref{rem:playcost-provable}.

\begin{proposition}[play-cost upper bound via query-hit mass]
\label{prop:playcostbound}
Let the true game be $M$ and let $\hat M$ be a deployed model whose contract functions agree with $M$ everywhere except on an error region $E$ (transition queries, or info-sets for the inference function). Fix any planner that is a deterministic function of its model's responses and a random seed, playing against a fixed (possibly stochastic) opponent and referee, and let $W(\cdot)$ be the expected game score of the agent deploying a model (win $=1$, draw $=\tfrac12$, loss $=0$). Let $\mu_{\mathrm{query}}(E)$ be the probability that the agent's search queries its model on $E$ at least once during a game. Then
\[
|W(M) - W(\hat M)| \;\leq\; \mu_{\mathrm{query}}(E);
\]
in particular $\mathrm{play\_cost} = W(M) - W(\hat M) \leq \mu_{\mathrm{query}}(E)$.
\end{proposition}

\begin{proof}
Couple the two deployments on identical seeds (planner, opponent, referee). Until the first time the search queries the model on $E$, every model response is identical in the two runs, hence so is every planner action, every referee transition, and every opponent response --- the runs are indistinguishable, so the event ``the search queries $E$'' has the same probability under both and $\mu_{\mathrm{query}}(E)$ is well defined on the coupled process. On sample paths where the search never queries $E$ the two games are identical and contribute zero to the difference; on the complementary event the score difference is at most $1$. Taking expectations gives the bound. Note the relevant distribution is the planner's \emph{query} distribution --- MCTS consults its model on imagined states --- which dominates the played-trajectory hit distribution: this is the formal counterpart of measuring adequacy on the search distribution.
\end{proof}

\begin{remark}[how much of play\_cost is provable]
\label{rem:playcost-provable}
Three fronts, of which only the last has a fundamental wall. \emph{(i) Upper bounds are theorems.} Proposition~\ref{prop:playcostbound} makes the danger law an end-to-end \emph{upper} bound: $\mathrm{danger} \leq \mu_{\mathrm{query}}(E) \times (1-r)^N$, an inequality with no fitted constant --- $(1-r)^N$ is exact and $\mu_{\mathrm{query}}$ is a property of the deployment distribution (measurable, or boundable by the reach-ratio machinery of Section~\ref{sec:coverage-bound}). Consistency check on the deployed instrument: competent trajectory reach of the cap region at the deployed cap (100) is $0.275$ (Figure~\ref{fig:mechanism}), a lower bound on $\mu_{\mathrm{query}}$, and the measured $\mathrm{play\_cost} = 0.091~[0.065, 0.117]$ respects it; if queries roughly track trajectory reach, the omitted rule flips the outcome in about a third of the games that reach the region --- pivotality, which no upper bound can supply (on Beacon $\mu_{\mathrm{query}}(D) = 1$ and the bound is vacuous while the true cost is $\tfrac12$). \emph{(ii) On solvable witnesses, play\_cost is exact.} Beacon's $\mathrm{play\_cost} = \tfrac12$ \emph{exactly} (Proposition~\ref{prop:beaconexact}), by exhaustion over its four deals --- so on the inference axis the danger law's predictions are fully analytic. \emph{(iii) Lower bounds are certifiable by witness.} A lower bound needs no game-solving: exhibiting an explicit opponent strategy and measuring the deficit is a statistical certificate. Our $n=4800$ measurement is exactly such a witness certificate: $\mathrm{play\_cost} \geq 0.065$ at seed-clustered $95\%$ confidence against the truth-planner opponent. What remains empirical at scale is only the \emph{exact} constant, which equals a game-value difference --- i.e., solving the game --- the play-value analogue of the enumeration/sampling wall of Remark~\ref{rem:threelevels}: exact where the game is solvable (Beacon, tic-tac-toe), bounded from above and below everywhere.
\end{remark}

\subsection{Measured danger curve}

We measure $\mathrm{play\_cost}$ precisely once ($\mathrm{play\_cost} \approx 0.09$--$0.13$ depending on sample size; the headline \texttt{play\_cost\_ci.py} run returns 0.091 at cap=100, n=4800 (20 seeds), 600 sims, with independent runs at smaller samples bracketing it: 0.131 at n=600, \texttt{play\_cost.py} 0.117--0.121 at n=240, \texttt{law\_sweep} 0.112 at cap=30) and sweep $\mathrm{rarity}$ cheaply by varying the ply cap (a lower cap makes the cap-and-equal-material event more common, hence a larger firing rate $r$; a higher cap makes it rarer). Rarity per cap is measured over 3000 random games. Results (Table~\ref{tab:dangercurve}, plotted in Figure~\ref{fig:dangerlaw}a):

\begin{table}[ht]
\centering
\small
\begin{tabular}{rrrrrr}
\toprule
cap & rarity & $(1-r)^{40}$ & danger@$N$=20 & danger@$N$=40 & danger@$N$=80 \\
\midrule
25  & 0.3367 & 0.0000 & 0.000 & 0.000 & 0.000 \\
40  & 0.2080 & 0.0001 & 0.001 & 0.000 & 0.000 \\
60  & 0.1073 & 0.0107 & 0.012 & 0.001 & 0.000 \\
80  & 0.0560 & 0.0997 & 0.038 & 0.012 & 0.001 \\
100 & 0.0253 & 0.3583 & 0.072 & 0.043 & 0.015 \\
120 & 0.0113 & 0.6339 & 0.096 & 0.076 & 0.048 \\
140 & 0.0067 & 0.7652 & 0.105 & 0.092 & 0.070 \\
\bottomrule
\end{tabular}
\caption{Measured danger curve: a threshold law in rarity. Computed with a round constant $\mathrm{play\_cost} = 0.12$ (headline measurements span $0.091$--$0.131$ across sample sizes); the threshold shape is insensitive to this choice.}
\label{tab:dangercurve}
\end{table}

\begin{figure}[ht]
\centering
\includegraphics[width=\textwidth]{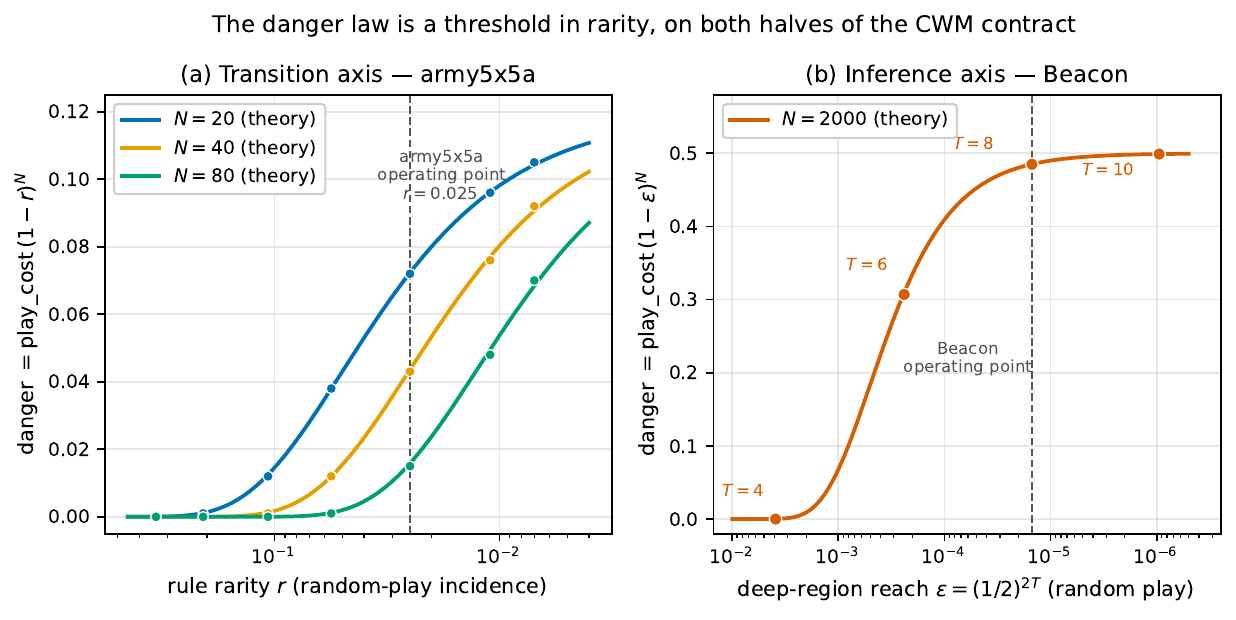}
\caption{The danger law $\mathrm{danger} = \mathrm{play\_cost}\,(1-r)^N$ is a \emph{threshold} in rarity, and the same law recurs on both halves of the CWM contract. Solid curves are the analytic law; markers are the measured operating points (Table~\ref{tab:dangercurve} for panel (a), Table~\ref{tab:inferenceaxis} for panel (b)). \textbf{(a)} Transition axis (army5x5a + material-at-cap), one curve per gate size $N\in\{20,40,80\}$: larger $N$ pushes the threshold toward rarer rules. The deployed army5x5a instrument ($r=0.0253$) sits in the danger zone. \textbf{(b)} Inference axis (Beacon), fixed gate $N=2000$, swept over the deep-region reach $\varepsilon=(1/2)^{2T}$: harm climbs from $\approx 0$ at $T=4$ to the full play cost by $T\geq 8$. The $x$-axes are reversed (rarer to the right) so danger rises rightward in both panels.}
\label{fig:dangerlaw}
\end{figure}

The result is a threshold law in rarity. Danger is approximately zero while the rule is common enough for a size-$N$ gate to catch it (cap $\leq 50$), rises through a threshold as the rule becomes rare (cap 60--100), and plateaus at approximately the full play\_cost once the rule almost always escapes the gate (cap $\geq 120$). The gate size $N$ shifts the threshold: larger $N$ pushes it toward rarer rules.

\subsection{Why Connect Four lies safely below the threshold}

Recall the six rules tested across Connect Four and army5x5a (Section~\ref{sec:rare-rule}). All of Connect Four's consequential rules have rarity 0.12--0.38, giving $(1-r)^{40} \approx 0$ --- they are caught by even a modest gate. army5x5a's material-at-cap rule at cap=100 has material-terminal rarity 0.0253, giving $(1-r)^{40} \approx 0.36$ --- deep in the danger zone. The structural reason is the same as before: in Connect Four, any rule a planner can force also appears regularly under random play (the rarity$\leftrightarrow$consequence tension); in army5x5a, competent play drives the game to the deep-ply-cap region that random play almost never reaches.

\section{Repairing the Gap: Translation, Not Inference}
\label{sec:repair}

If the gap is caused by a missing rule, can it be repaired by providing example transitions that demonstrate the rule? We ran a systematic set of repair attempts on army5x5a + material-at-cap, synthesized from incomplete rules.

\subsection{Repair experiments}

Table~\ref{tab:repair} reports every repair condition. All conditions use the mini synthesizer unless noted; play winrate is vs the true game, 40 arena games at 400 simulations (30 for the DAgger rows), one synthesis seed per condition; baseline 0.28, fair truth-vs-truth 0.50. Discriminating examples are transitions that involve the material-at-cap rule. The proper-DAgger condition follows Ross et al.'s dataset-aggregation scheme: each round self-plays the current (flawed) CWM, labels its game-path states with the truth, and adds them to a growing dataset (0, 3, 5, 11 discriminating transitions accumulated over three rounds); the rule is never learned despite the growing on-distribution evidence. \textsuperscript{$\dagger$}One of the three rounds' synthesis collapsed at the gate (accuracy $0.00$); the other two reached only $0.994$--$0.997$, below $1.0$ precisely because the aggregated dataset now contains rule-bearing transitions the rule-blind CWM cannot match --- the same signal as the sweep of Section~\ref{sec:repair}. (In the artificial-states condition, constructed states were validated pre-move, so a random action can equalize material or capture the general: of the 120 constructed transitions, 115 demonstrate the rule decisively --- 4 end in general capture and 1 in an equal-material draw. The harvested-real condition validates post-hoc and its count is exact.) ``Rule learned'' is judged by code inspection (does the synthesized module contain the material-at-cap branch?) cross-checked against play parity with the complete-rules control; it is a binary structural fact about the synthesized code, not a winrate threshold. With one seed and 40 games per condition, each individual winrate carries a wide interval (e.g.\ $0.28$ is Wilson 95\% $[0.16, 0.43]$); the load-bearing signal is therefore not any single cell but the \emph{consistent} no-learn outcome across every repair form together with the clean parity of the complete-rules control --- not a precise per-condition effect size.

\begin{table}[ht]
\centering
\small
\resizebox{\textwidth}{!}{%
\begin{tabular}{lcccc}
\toprule
Repair attempt & discr.\ examples & gate acc & rule learned & winrate \\
\midrule
none (random trajectories) & 0 & 1.000 (false security) & no & 0.28 \\
naive DAgger (dump competent trajectories) & ${\sim}2$ & 0.9996 & no & 0.28 \\
proper DAgger (iterated, dataset aggregated across rounds) & 3--11 aggregated & 0.994--0.997\textsuperscript{$\dagger$} & no & 0.28--0.42 \\
targeted, artificial states & 120 (115 decisive) & mini 0.916 / large 0.004 & no & mini 0.35 / large 0.05 \\
targeted, real (harvested on-manifold) & 54 & mini 0.959 / large 0.959 & no & mini 0.35 / large 0.42 \\
COMPLETE rules + targeted (control) & 120 & 1.000 (0 iters) & yes & 0.53 \\
\bottomrule
\end{tabular}%
}
\caption{Repair attempts on the incomplete-rules CWM. Detection works; repair does not.}
\label{tab:repair}
\end{table}

\subsection{Findings}

\paragraph{Detection works, repair does not.} Verifying on the play/search distribution rather than on random trajectories drops the gate below 1.0, detecting the inadequacy that random-trajectory verification missed. But neither mini nor large can \emph{infer} the missing rule from examples. Even 54 real, on-manifold discriminating transitions with 12 refinement iterations leave the gate at 0.959 and the rule unlearned.

\paragraph{Spec completeness, not code-writing ability, is what binds.} Given the rule in \texttt{RULES\_TEXT}, the model encodes it correctly in 0 refinement iterations and plays at parity with the baseline (0.53 $\approx$ 0.50). The complete-rules control isolates the cause: the limitation is not the synthesizer's code-writing ability but the absence of the rule from the specification.

\paragraph{Scale helps marginally but not sufficiently --- and only where there is signal to help.} The large model (0.42) exceeds mini (0.35) on real on-manifold data, but both remain far below the complete-rules baseline (0.53). The inference ceiling is not a mini-specific artifact: the marginal gain lives in the on-manifold regime, where the discriminating transitions are actually sampled and the task is \emph{generalizing} observed rule-firings into code. In the unsampled regime the scale effect does not merely shrink --- it is provably absent (Section~\ref{sec:sample-identifiability}): no model, at any scale, can prefer the correct rule from a sample that carries no evidence for it.

\paragraph{Off-manifold repair data corrupts synthesis.} Artificial (unreachable) discriminating states cause catastrophic failure in the large model (accuracy collapses to 0.004, win rate 0.05). The synthesizer attempts to fit transitions that cannot arise in real play and damages the parts of the CWM it had already learned correctly.

\subsection{Sample-identifiability: the provable core of ``translation, not inference''}
\label{sec:sample-identifiability}

The repair experiments above are empirical and scoped to the GPT-5.x family. They are, however, sitting on top of a result that is \emph{provable and universal} --- and stating it does not soften ``translation, not inference'' but strengthens it, by separating a theorem from a conjecture.

Let a rule $R$ be omitted from the specification (the incomplete-rules / no-rules regime), so the only possible source for it is the trajectory sample. Let $R$ fire on a set of play-throughs whose probability is $r$ under the verification (uniform-random) distribution $D$. Write $M_R$ for the world model \emph{with} $R$ and $M_\emptyset$ for the otherwise-identical model \emph{without} $R$; they agree on every transition except those in the rule region. The pipeline draws $N$ i.i.d.\ play-throughs from $D$ (the training sample, which is also the gate --- see ``What $N$ counts,'' Section~\ref{sec:danger-law}).

\begin{proposition}[sample-identifiability]
\label{prop:identifiability}
Condition on the event $\neg E$ that no sampled play-through hits the rule region; $P(\neg E) = (1-r)^N$. On $\neg E$, $M_R$ and $M_\emptyset$ produce identical outputs on every sampled transition, so the sample is observationally equivalent under the two models: any score that depends on a candidate model only through its outputs on the sampled transitions (the gate's transition accuracy, the likelihood of the sample, the data-fit term of a program-search objective) is identical for $M_R$ and $M_\emptyset$. Hence the sample carries no evidence favouring the correct model --- the gate cannot distinguish or reject $M_\emptyset$, and any preference for $M_R$ must come from a prior or from the specification, not from the sampled transitions. The omitted rule is in this sense unidentifiable from the sample, and a sample-passing $M_\emptyset$ is admissible.
\end{proposition}

\begin{proof}
On $\neg E$ the two models are pointwise equal on every transition the sample contains, so any score that reads a model only through its outputs on the sampled transitions takes the same value on both. Such a score therefore cannot order $M_R$ and $M_\emptyset$; any ordering a procedure imposes is a prior or a tie-break (e.g.\ a description-length term), not evidence carried by the sample.
\end{proof}

\begin{corollary}[unidentifiability probability = the danger-law gate-miss factor]
The probability that the omitted rule is unidentifiable from the $N$ samples is exactly $(1-r)^N$ --- the same factor as the danger law (Section~\ref{sec:danger-law}). The danger law is thus an identifiability statement: $\mathrm{danger} = \mathrm{play\_cost} \times P(\text{rule unidentifiable from the } N \text{ samples})$. The gate-passing-but-rule-blind seeds of Section~\ref{sec:rare-rule} are \emph{consistent with} this event ($\neg E$) --- a property of the data rather than necessarily an LLM failure. We did not log per seed whether the rule region appeared. By contrast, in the resampling runs of Section~\ref{sec:repair} below, the gate accuracy shows the rule often \emph{was} present and was still not learned.
\end{corollary}

This splits ``translation, not inference'' cleanly into a provable core and an empirical residual.

\begin{itemize}
\item \textbf{(a) Provable, universal.} When the rule is absent from the sample (probability $(1-r)^N$), the sample provides no evidence to infer it --- for \emph{any} learner; this is a property of the data, not a limitation of LLMs (Proposition~\ref{prop:identifiability}). The gap-exhibiting seeds in Section~\ref{sec:rare-rule} are consistent with this event. This part is a theorem.
\item \textbf{(b) Empirical, LLM-specific.} Even when the rule \emph{is} present in the sample (probability $1-(1-r)^N$), the LLM synthesizer does not reliably encode it: the Section~\ref{sec:repair} repair battery (proper DAgger, targeted on-manifold examples) supplies discriminating transitions and the rule is still omitted. This is the genuinely LLM-specific, scoped finding --- established under the GPT-5.x family (mini, large) across the Section~\ref{sec:repair} regimes. Its universal-for-all-models form remains a conjecture.
\end{itemize}

We measure (b) directly on the synthesis pipeline (\texttt{scripts/danger\_synthesis\_sweep.py}): synthesizing from the incomplete rules and $N$ true-game trajectories, with refinement drawing \textbf{fresh} trajectories each iteration, the fraction of seeds whose CWM is rule-blind is essentially $1$ regardless of $N$ for both model sizes (Table~\ref{tab:synthcurve}):

\begin{table}[ht]
\centering
\small
\begin{tabular}{cccc}
\toprule
$N$ (games, initial batch) & mini rule-blind & large rule-blind & initial-batch floor $(1-r)^N$ \\
\midrule
40  & 20/20 = 1.00 & 20/20 = 1.00 & 0.358 \\
120 & 20/20 = 1.00 & 20/20 = 1.00 & 0.046 \\
200 & 20/20 = 1.00 & 20/20 = 1.00 & 0.006 \\
\bottomrule
\end{tabular}
\caption{Synthesis-pipeline danger curve: both model sizes remain rule-blind far above the information-theoretic floor. 20 seeds per cell (Wilson 95\% lower bound on the rule-blind rate: 0.839); no synthesis crashed in any of the 120 runs, so every denominator counts CWMs that ran and were scored.}
\label{tab:synthcurve}
\end{table}

\textbf{On the floor.} Two identifiability floors apply, and they bracket the result. The $(1-r)^N$ column is the \emph{initial-batch} floor --- the chance the rule is absent from the first $N$ games. But because refinement draws \textbf{fresh} batches each iteration, a run that refines for $k$ iterations is exposed to $\approx N(1+k)$ distinct games, so the \emph{full-pipeline} exposure floor is $\approx (1-r)^{N(1+k)}$, which for the refined seeds (up to $\sim$1400 games at $N=200$, $k=6$) is effectively $0$ --- i.e.\ the rule is present in the pipeline with probability $\approx 1$. The gate accuracy makes this concrete: at $N=200$, \textbf{no} large seed reaches gate $1.0$ (0/20), meaning the rule-bearing transitions \emph{are} in every sample (the rule-blind CWM mismatches them, so the gate cannot reach $1.0$), yet after six refinement iterations all 20 CWMs are still rule-blind. So the result is not an instance of the unidentifiability event $\neg E$: the rule is present ($E$ holds) and simply not learned --- the pure (b) residual. (Some seeds do reach gate $1.0$ while blind --- 6/20 mini and 3/20 large at $N=120$ --- which is the batch-miss event: refinement stops as soon as accuracy hits $1.0$ on the \emph{current} batch, and with a per-batch miss rate of $(1-r)^{120} = 0.046$ compounded over up to seven fresh batches, stopping on a rule-free batch has probability $\approx 0.28$ --- matching the observed 6/20. The $N=200$ large column, where no seed reaches the gate, rules that explanation out for the headline cell.) Note the contrast with the main pipeline, where refinement re-checks a fixed trajectory set (Section~\ref{sec:gate}): under fresh resampling the rule had every opportunity to appear and still was not learned.

\textbf{A feedback-channel confound, found and closed.} Auditing the harness surfaced a confound in the statement above: the refinement feedback truncated every failure line at 200 characters and never included the expected values --- the header said ``expected vs got'' while the body carried neither --- and the synthesis prompt shows only the \emph{first} 30 transitions of the trajectory stream --- a deterministic prefix that can \emph{never} contain a material-at-cap terminal, whose global index is $\geq 99$ by construction. So although the rule was present in the \emph{gate's} trajectories, it almost never reached the \emph{model} in legible form: the sweep above measures the pipeline as deployed, but by itself cannot attribute the blindness to the model rather than to the harness's feedback channel. We fixed the channel (failure lines now show \texttt{expected=} and \texttt{got=} for each mismatched field) and reran the headline cell as a discriminant experiment (\texttt{scripts/refine\_feedback\_cell.py}): with the rule's expected returns \emph{printed in the feedback} --- e.g.\ \texttt{returns: expected=\{1: 1.0, 2: -1.0\} got=\{1: 0.0, 2: 0.0\}} on cap states with unequal material --- the result is \textbf{12/12 rule-blind} (6 mini + 6 large seeds, zero crashes). The signature is sharper than in the deployed-pipeline sweep: in 9 of the 12 seeds the model repairs everything \emph{except} the rule, stalling at gate accuracy ${\approx}0.999$ with exactly the material transitions unfixed (the remaining three end at low gate accuracy, 0.11--0.59 --- refinement collapse, a different failure mode --- and are rule-blind too). The (b) residual is therefore not an artifact of the feedback channel: shown the discrepant expected values directly, the model corrects the dynamics around them and still does not posit the rule.

So ``translation, not inference'' should be read as the conjunction: \textbf{(a)} an omitted rule is unidentifiable from a sample that misses it (provable, universal, exactly the danger-law gate-miss event), and \textbf{(b)} even a sample that \emph{contains} it does not reliably teach the LLM (empirical, tested on GPT-5.x mini/large across the Section~\ref{sec:repair} regimes). Stating (a) as a theorem strengthens the claim: the part of the failure that looks like an LLM weakness is in fact information-theoretic and binds every possible learner; only the residual (b) is a model-specific finding, and we are explicit that its universal form is conjectural.

\subsection{Conclusion: translation, not inference}

Under the tested regimes, LLM CWM synthesis behaves as \emph{rule translation}: it correctly encodes rules it was given and, across the two model sizes tested (mini, large) and every data regime we ran, did not infer the omitted rule, even when that rule was demonstrated by example transitions. Beneath this empirical finding sits a universal one (Section~\ref{sec:sample-identifiability}, Proposition~\ref{prop:identifiability}): when the omitted rule is absent from the sample --- probability $(1-r)^N$, the danger-law gate-miss event --- no learner can recover it \emph{from those sampled transitions alone}; any recovery must come from a prior or the specification, not from the sample. The actionable implication is that the specification must be complete before synthesis, and that verification on the play distribution detects incompleteness but does not repair it. Feeding example transitions is not a reliable substitute for a complete specification.

\section{Imperfect Information: The Inference Function as a New Failure Surface}
\label{sec:imperfect}

The danger law applies not just to the transition function (the CWM's model of how states evolve) but also to the inference function (the CWM's model of how to reconstruct hidden state from observations). We extend the contract, prove a coverage bound that explains when the inference gate is provably safe, and construct a minimal game where it is not.

\subsection{Pipeline validation: Kuhn poker}
\label{sec:kuhn-validation}

Before constructing a gap, we validate the imperfect-information pipeline on Kuhn poker, a well-understood minimal game (3-card deck, 1 betting round per player, net-chip payoff; Table~\ref{tab:kuhn}).

\begin{table}[ht]
\centering
\small
\resizebox{\textwidth}{!}{%
\begin{tabular}{lllll}
\toprule
Synth & transition gate & inference gate (obs / infer) & CWM-vs-truth play & fair baseline \\
\midrule
large & 1.000 (0 iters) & 1.000 / 1.000 & 0.470 [0.422, 0.519] & 0.470 [0.422, 0.519] \\
mini  & 1.000 (0 iters) & 0.500 / 1.000 & 0.470 [0.422, 0.519] & 0.470 [0.422, 0.519] \\
\bottomrule
\end{tabular}%
}
\caption{Kuhn poker pipeline validation.}
\label{tab:kuhn}
\end{table}

Both model sizes pass the transition gate and synthesize an exact \texttt{infer\_states} (inference rate 1.000), and both play at parity with the truth-vs-truth baseline --- consistent with a near-zero gap, as expected when the model has the game (large recalls Kuhn; mini reconstructs it from the specification). The mini synthesis differs from the ground truth only in an \emph{observation convention}: it places player 2's own hidden card at a different index of the observation vector (observation rate 0.500 --- every player-2 observation mismatches by this relabeling), a wrong-but-running convention divergence that leaves play unaffected.

\subsection{When the inference gate is provably sufficient: a coverage bound}
\label{sec:coverage-bound}

We now formalize when the inference gate, sampled on random play, is \emph{identifying} --- that is, when every competent-play-relevant inference error would be caught by the random-trajectory gate.

\paragraph{Setup.} Consider a finite two-player extensive-form game \textbf{with perfect recall}, chance (the deal) and imperfect information. We use the following notation:
\begin{description}
\item[$b$:] maximum, over \textbf{player} information sets $I$, of $|A(I)|$ (the number of actions available at $I$), with chance handled separately through $p_{\text{chance}}$;
\item[$d(I)$:] number of player-action edges on a shortest history reaching information set $I$;
\item[$\dmax$:] $\max_I d(I)$, the maximum such depth;
\item[$p_{\text{chance}}$:] minimum probability of a deal consistent with any reachable info-set;
\item[$\Ireach$:] set of reachable info-sets;
\item[$\pi^\sigma(\cdot)$:] realization (reach) probability of a history or info-set under strategy profile $\sigma$;
\item[$\mathrm{supp}(\cdot)$, $\mathrm{reach}(\cdot)$:] the histories, resp.\ info-sets, given positive probability.
\end{description}
The uniform-random policy $\rho$ plays every legal action with probability $1/|A(I)| \geq 1/b$, and therefore assigns positive probability to every legal action (full support).

\begin{lemma}[full-support inclusion]
\label{lem:fullsupport}
Because $\rho$ assigns positive probability to every legal action, $\mathrm{supp}(\pi^\sigma) \subseteq \mathrm{supp}(\pi^\rho)$ for every profile $\sigma$; equivalently $\mathrm{reach}(\sigma) \subseteq \mathrm{reach}(\rho)$.
\end{lemma}

\begin{proof}
This is the standard fact that a fully-mixed strategy reaches every node reachable under any profile. Reach of an info-set is taken under the actual interactive profile (planner + opponent + chance); chance edges are shared and $\rho$ dominates each player's per-edge contribution ($\rho$-probability $\geq 1/b > 0$ on every player edge), so any history with $\pi^\sigma(h) > 0$ has $\pi^\rho(h) > 0$.
\end{proof}

\begin{lemma}[reach lower bound under $\rho$]
\label{lem:reachbound}
Every reachable info-set $I$ has $\pi^\rho(I) \geq p_{\text{chance}} \cdot b^{-d(I)} \geq p_{\text{chance}} \cdot b^{-\dmax}$.
\end{lemma}

\begin{proof}
Take a shortest history $h \in I$, i.e.\ one attaining $d(h) = d(I)$. Along $h$ each player edge at info-set $I_t$ has $\rho$-probability $1/|A(I_t)| \geq 1/b$, so the \textbf{realization probability} of $I$ (in the sense of \citeauthor{vonstengel1996sequenceform}'s \citeyearpar{vonstengel1996sequenceform} sequence form) satisfies $\pi^\rho(h) = \pi_{\text{chance}}(h) \cdot \prod_t 1/|A(I_t)| \geq p_{\text{chance}} \cdot b^{-d(h)} = p_{\text{chance}} \cdot b^{-d(I)}$, and $\pi^\rho(I) \geq \pi^\rho(h)$.
\end{proof}

\begin{theorem}[the inference gate is identifying when $N \gtrsim b^{\dmax}$]
\label{thm:coverage}
Assume \emph{detectability}: whenever the gate visits an info-set on which \texttt{infer\_states} errs, its comparison against ground truth at that info-set surfaces the error. Then for $N \gtrsim b^{\dmax} \cdot p_{\text{chance}}^{-1} \cdot \log |\Ireach|$, no gate-passing inference function can be play-inadequate through a coverage gap. The argument: draw $N$ i.i.d.\ games under $\rho$; the probability that some reachable info-set is never visited is at most $|\Ireach| \cdot \exp(-N \cdot p_{\text{chance}} \cdot b^{-\dmax})$ (union bound over info-sets plus Lemma~\ref{lem:reachbound}), so at this $N$ the random sample covers every reachable info-set with high probability --- and by Lemma~\ref{lem:fullsupport}, every info-set any policy (including a competent planner) relies on. An inference function whose error is confined to reachable info-sets is then detected with high probability under detectability.
\end{theorem}

\begin{proof}
Immediate from the union bound over $\Ireach$ and the per-info-set reach lower bound of Lemma~\ref{lem:reachbound}; the coverage of equilibrium-relevant info-sets follows from Lemma~\ref{lem:fullsupport}. Detectability holds for our inference gate by construction: at every visited observation history it compares the synthesized consistent-state set against the ground-truth set elementwise, so any discrepancy at a visited info-set is recorded as a mismatch. It is not automatic for an arbitrary gate --- one that checked only a coarse summary of \texttt{infer\_states} could visit an erring info-set without surfacing the error --- which is why we state it as a hypothesis.
\end{proof}

\begin{remark}[equilibrium-robustness]
The sufficiency direction (coverage $\Rightarrow$ identifying) does not depend on the reference distribution being MCTS reach. Because $\rho$ has full support, $\mathrm{reach}(\sigma^*) \subseteq \mathrm{reach}(\rho)$ for the Nash / best-response profile $\sigma^*$ exactly as for any other profile (Lemma~\ref{lem:fullsupport}). The bound therefore certifies coverage of every info-set that \emph{equilibrium} play relies on, not merely those the deployed planner happens to visit --- a strength rather than a caveat. Substituting equilibrium reach for MCTS reach would change which info-sets count as relevant (and hence the numbers) but not the sufficiency argument.
\end{remark}

\begin{corollary}[Kuhn provably covered; Leduc's sampled competent-visited subset covered]
Carrying the exact constants --- by enumerating the reachable info-sets and their reach probabilities under uniform-random play (\texttt{scripts/coverage\_bound\_constants.py}) --- upgrades Kuhn from an order statement to a provable null, and certifies the \emph{sampled} competent-visited subset of Leduc.

\emph{Kuhn.} Exact enumeration gives $|\Ireach| = 12$ and minimum reach probability $\pi_{\min} = 0.083$, so the tight sufficient sample size is $N_{\text{suff}} = 66$. At the deployed gate $N = 80$, the union-bound upper bound on the coverage-failure probability (computed from the exact reach probabilities) is $0.0028$ --- Kuhn is therefore \textbf{provably covered}.

\emph{Leduc.} Exact enumeration gives $|\Ireach| = 576$ and $\pi_{\min} = 3.5 \times 10^{-4}$. The worst-case-depth bound needs $N \approx 818\text{k}$ (with $d_{\max} = 6$: the round-1 observation merge shortens shortest histories, so Leduc's $d_{\max}$ is 6, not its 8-move horizon) and even the tight $\pi_{\min}$ bound needs $N \approx 27\text{k}$, so the theorem does \textbf{not} certify full coverage of all 576 reachable info-sets at the deployed $N = 8000$ --- we say so honestly. However, our coverage claim concerns \emph{only} the info-sets a competent planner relies on. We \emph{sample} that subset with 200 determinized-MCTS self-play games (\texttt{scripts/coverage\_competent\_leduc.py}) --- this enumerates the competent-visited info-sets we observed, not provably the entire competent support. The 146 sampled competent-visited info-sets have $\pi_{\min} = 6.9 \times 10^{-4}$, and at $N = 8000$ the union-bound upper bound on the coverage-failure probability over this subset (computed from exact random-reach probabilities) is $0.027 < 0.05$, so the \textbf{sampled} competent-visited subset is covered with high probability. Combined with the measurement (none of the 1259 competent visits lands on an info-set random play missed), this upgrades ``empirically covered'' to ``covered under exact random-reach probabilities on the sampled competent-relevant subset'' --- short of a guarantee over the full competent support, which would require enumerating it.

\emph{Equilibrium-reach upgrade.} MCTS self-play is a proxy for competent reach; the normatively correct reference is equilibrium reach (Remark~\ref{rem:reference-dist}). We therefore ran an equilibrium solver on the true games --- full-tree CFR+ over the CWM contract (\texttt{cwm.cfr}; external-sampling MCCFR is also implemented, and the module is validated on Kuhn against the analytic game value $-1/18$ \citep{kuhn1950poker}) --- and recomputed the coverage question against the resulting profile $\bar\sigma$ (\texttt{scripts/equilibrium\_coverage.py}). On Leduc, $\bar\sigma$ reaches game value $-0.0866$ with exploitability $0.084$ chips (as $\mathrm{BR}_1 + \mathrm{BR}_2$, i.e.\ NashConv; halve it under the divide-by-two convention) after 1{,}000 iterations (decreasing monotonically; an approximate equilibrium, reported as such). One technical subtlety matters and we flag it: Leduc's \emph{instantaneous} observation merges distinct betting histories (the per-round counters reset between rounds), so the observation-keyed game has imperfect recall, where CFR carries no guarantee; the solver therefore runs on perfect-recall keys (observation plus exact public history) and the resulting reach is projected onto the observation keys the gate actually samples. The results: the \textbf{equilibrium-weighted uncovered mass} --- the expected fraction of $\bar\sigma$'s decision mass falling on info-sets a size-$N$ random gate never visits --- is $0.013\%$ at the deployed $N=8000$ (Kuhn: $0.015\%$ at $N=80$), and the union bound certifies coverage of every info-set with equilibrium reach $\geq 10^{-3}$ (316 info-sets, failure probability $0.021$). The negligible-reach tail (info-sets below $10^{-3}$) is not union-bound-certifiable at $N=8000$, exactly as for the full 576; the mass statement says it carries $\approx 10^{-4}$ of equilibrium play. The conclusion is unchanged and strengthened: the info-sets that \emph{equilibrium} play relies on are covered by the random gate, and this is robust to the profile's quality (a much cruder profile, exploitability $0.79$, yields the same coverage conclusion).
\end{corollary}

\paragraph{An enumeration-free companion certificate: bounding the error mass.}
Theorem~\ref{thm:coverage} certifies the strongest possible conclusion --- every reachable info-set is visited, so \emph{any} error pattern confined to reachable info-sets is caught --- but instantiating its constants requires enumerating $\Ireach$ and the exact reach probabilities. That is feasible for Kuhn and Leduc (the corollary above) and intractable in general; indeed the Leduc guarantee already stops at the \emph{sampled} competent subset precisely because the full competent support would have to be enumerated. We therefore add a companion certificate that trades conclusion strength for generality: it certifies not that every info-set was visited, but that the \emph{undetected error region of the accepted artifact has small sampling mass}. Its constants involve no enumeration --- no $|\Ireach|$, no $\pi_{\min}$, no reach probabilities --- so it applies verbatim to games of arbitrary size. The two results answer different questions and neither subsumes the other (Remark~\ref{rem:tworoutes}); we keep both.

Fix a candidate inference function $f$ and let $E_f \subseteq \Ireach$ be the set of reachable info-sets on which $f$ errs. For a per-game sampling distribution $\nu$ (a policy profile, or a mixture of profiles resampled per game), write
\[
\mu_\nu(E) \;=\; \Pr_{\text{game} \sim \nu}\!\big[\text{the game visits at least one info-set of } E\big]
\]
for the per-game \emph{hit probability} of $E \subseteq \Ireach$. Let $\bar d$ denote the \emph{player-move horizon} --- the maximum number of player-action edges on any complete history --- so $\bar d \geq \dmax$, strictly for Kuhn and Leduc ($\bar d = 3$ and $8$ vs $\dmax = 2$ and $6$: an info-set is reached before an action is taken from it, and Leduc's round-1 observation merge shortens shortest histories further).

\begin{theorem}[enumeration-free error-mass bound]
\label{thm:errormass}
Assume detectability, and let the gate draw $N$ i.i.d.\ games from $\nu$ and check every visited info-set.
\begin{itemize}
\item[(i)] (\emph{Held-out gate.}) If $f$ is fixed before the gate sample is drawn, then for any $\delta \in (0,1)$, with confidence $1-\delta$: if $f$ passes the gate, $\mu_\nu(E_f) \leq \ln(1/\delta)/N$.
\item[(ii)] (\emph{Class-uniform / Occam.}) If $f$ may instead depend on the gate sample --- as in a pipeline whose training trajectories double as the gate --- but is guaranteed to be a program of description length at most $\ell$ bits, then with confidence $1-\delta$, \emph{simultaneously} every gate-passing program of length $\leq \ell$ satisfies $\mu_\nu(E_f) \leq \big((\ell+1)\ln 2 + \ln(1/\delta)\big)/N$.
\end{itemize}
\end{theorem}

\begin{proof}
(i) If $\mu_\nu(E_f) > \varepsilon$, the probability that none of the $N$ games visits $E_f$ is $(1 - \mu_\nu(E_f))^N < e^{-\varepsilon N}$; under detectability any visit to $E_f$ surfaces a mismatch, so $e^{-\varepsilon N}$ upper-bounds the probability that $f$ passes. Setting $\varepsilon = \ln(1/\delta)/N$ makes this $\leq \delta$. (ii) There are fewer than $2^{\ell+1}$ binary programs of length $\leq \ell$; a union bound over the class multiplies the failure probability by $2^{\ell+1}$, and $2^{\ell+1} e^{-\varepsilon N} \leq \delta$ at $\varepsilon = ((\ell+1)\ln 2 + \ln(1/\delta))/N$. Sample-dependent selection is covered because the guarantee holds simultaneously over the class. Neither argument touches $|\Ireach|$ or any reach probability.
\end{proof}

\begin{corollary}[transfer to play: the reach ratio]
\label{cor:reachratio}
Let $E \subseteq \Ireach$ and let $\sigma$ be \emph{any} strategy profile.
\begin{itemize}
\item[(i)] (\emph{Pure-random gate}, $\nu = \rho$.) $\mu_\sigma(E) \leq b^{\bar d} \cdot \mu_\rho(E)$; combined with Theorem~\ref{thm:errormass}(i), any gate-passing $f$ has $\mu_\sigma(E_f) \leq b^{\bar d} \ln(1/\delta) / N$.
\item[(ii)] (\emph{Mixture gate.}) If each gate game is drawn from $\rho$ with probability $1-\lambda$ and from a reference competent profile $\hat\sigma$ (e.g.\ determinized-MCTS self-play on the ground truth, which the harness owns at gate time) with probability $\lambda$, then any gate-passing $f$ has $\mu_{\hat\sigma}(E_f) \leq \ln(1/\delta)/(\lambda N)$, while for arbitrary $\sigma$ the bound $\mu_\sigma(E_f) \leq b^{\bar d} \ln(1/\delta) / ((1-\lambda) N)$ is retained through the $\rho$ component (full support is preserved, so Lemma~\ref{lem:fullsupport} and the equilibrium-robustness remark survive).
\end{itemize}
In the danger-law reading, any play harm mediated by consulting $f$ on an erring info-set obeys $\mathrm{danger} \leq \mathrm{play\_cost} \times \mu_\sigma(E_f)$ with $\sigma$ the deployment profile.
\end{corollary}

\begin{proof}
(i) Let $H_E$ be the set of minimal histories entering $E$ (no proper prefix visits $E$); $H_E$ is prefix-free, so $\mu_\sigma(E) = \sum_{h \in H_E} \pi^\sigma(h)$. For each $h$, chance edges are shared and every player edge has $\sigma$-probability $\leq 1$ and $\rho$-probability $\geq 1/b$, so $\pi^\sigma(h) \leq b^{d(h)} \pi^\rho(h) \leq b^{\bar d} \pi^\rho(h)$; summing gives the claim. (ii) $\mu_\nu = (1-\lambda)\mu_\rho + \lambda \mu_{\hat\sigma}$, so $\mu_\nu \leq \ln(1/\delta)/N$ implies both $\mu_{\hat\sigma} \leq \ln(1/\delta)/(\lambda N)$ and $\mu_\rho \leq \ln(1/\delta)/((1-\lambda)N)$; apply (i) to the latter.
\end{proof}

\begin{corollary}[constants without enumeration]
Instantiating at $\delta = 0.05$ (\texttt{scripts/\allowbreak error\_\allowbreak mass\_\allowbreak certificate.py}; the only game-dependent inputs are the rule-level constants $b$ and $\bar d$):

\emph{Kuhn} ($N = 80$, $b = 2$, $\bar d = 3$): any gate-passing \texttt{infer\_states} fixed before the gate has undetected-error hit mass $\mu_\rho(E_f) \leq 0.037$, and $\mu_\sigma(E_f) \leq 8 \times 0.037 = 0.30$ for every profile $\sigma$. This is strictly weaker than the enumerative corollary (which certifies \emph{full coverage} at $N = 80$) --- where enumeration is feasible, Theorem~\ref{thm:coverage} remains the sharper tool.

\emph{Leduc} ($N = 8000$, $b = 3$, $\bar d = 8$): $\mu_\rho(E_f) \leq 3.7 \times 10^{-4}$, but the worst-case transfer factor $b^{\bar d} = 6561$ makes the any-profile bound vacuous --- consistent with Theorem~\ref{thm:coverage} declining to certify Leduc at this $N$. The mixture gate is what the enumeration route could not provide: at the same $N = 8000$ with $\lambda = 1/2$, any gate-passing $f$ has $\mu_{\hat\sigma}(E_f) \leq 7.5 \times 10^{-4}$ --- a certificate for the competent-play-relevant error mass with \emph{no} enumeration of the competent support, where the enumerative route needed $N \approx 27\text{k}$ even for the random-reach guarantee and could only certify the \emph{sampled} competent subset.
\end{corollary}

\begin{remark}[relation between the two certificates]
\label{rem:tworoutes}
Neither result subsumes the other. Theorem~\ref{thm:coverage} is a \emph{for-all} guarantee: once the sample covers $\Ireach$, every error pattern confined to reachable info-sets is caught, including one chosen adversarially after the sample; its price is constants ($\pi_{\min}^{-1} \log |\Ireach|$) that must be computed by enumeration and grow with the game. Theorem~\ref{thm:errormass} certifies only the accepted artifact (or a description-length class), and always leaves an $\varepsilon$-mass of undetected error; its price is a hypothesis about how $f$ was chosen --- but its constants are game-size-free. In practice: use the coverage route where enumeration is feasible (Kuhn: provably covered, a conclusion Theorem~\ref{thm:errormass} cannot reach at any $N$), and the error-mass route everywhere else. On the choice of hypothesis in Theorem~\ref{thm:errormass}: for kB-scale synthesized artifacts ($\ell \sim 10^4$ bits) the class-uniform constant requires $N \gtrsim \ell$, exceeding every gate deployed in this paper, so the practical reading is (i) --- draw the gate sample \emph{after} synthesis (a held-out gate), which is cheap since the harness owns the reference implementation. This is also consonant with Section~\ref{sec:sample-identifiability}: when the training sample doubles as the gate, a compact wrong hypothesis consistent with the whole sample is exactly the case that the fixed-candidate argument does not exclude.
\end{remark}

\begin{remark}[tightness: Beacon realizes the reach ratio]
The $b^{\bar d}$ transfer factor in Corollary~\ref{cor:reachratio} is not slack in the analysis. In Beacon (Section~\ref{sec:beacon}), the deep region $D$ has $\mu_\rho(D) = (1/2)^{2T}$ while optimal play reaches it with probability 1 --- the reach ratio $2^{2T}$ is realized, matching $b^{\bar d}$ up to the guess rounds. So for a pure-random gate the exponential transfer factor is unavoidable, and the mixture instantiation of Corollary~\ref{cor:reachratio}(ii) is the only lever that removes it --- which is precisely this paper's recommendation to verify on the search distribution, now in certificate form. The mixture certifies the reference profile $\hat\sigma$ at cost $1/\lambda$; a guarantee uniform over \emph{all} competent policies still pays $b^{\bar d}$ through the $\rho$ component, and Beacon shows this too cannot be improved.
\end{remark}

\begin{corollary}[when a gap is possible]
A coverage gap requires $b^{\dmax} \gg N$ at feasible $N$ --- large branching and/or large depth --- with a competent policy that concentrates reach on a deep region of $\rho$-measure $\ll 1/N$. In game-theoretic terms, the gap lives on info-sets reached with negligible probability under the \textbf{sampling policy} but on-path under \textbf{optimal play} --- off-equilibrium-path-style info-sets that the verification distribution does not constrain (it places negligible sampling mass there, so errors there go unpenalized). This is the imperfect-information analogue of the rare-rule condition: a region that random play almost never samples but competent play reliably visits. Theorem~\ref{thm:errormass} makes the requirement quantitative: a gate-passing artifact's error region has sampling hit mass $\leq \ln(1/\delta)/N$, so a gap requires the deployment policy to concentrate constant hit probability on a region of sampling hit mass below $\ln(1/\delta)/N$ --- a realized reach ratio $\gtrsim N$.
\end{corollary}

\begin{remark}[three levels of conclusion, and where enumeration actually matters]
\label{rem:threelevels}
Summarizing this subsection, it helps to separate three levels of guarantee, because ``requires enumeration'' and ``requires infeasible $N$'' are different obstacles and only one of them is fundamental. \emph{(1) Per-artifact guarantees} --- is \emph{this} accepted inference function safe to deploy? Theorem~\ref{thm:errormass} answers this definitively for games of arbitrary size with no enumeration: undetected error mass $\leq \ln(1/\delta)/N$ under the gate's sampling distribution, and $\leq \ln(1/\delta)/(\lambda N)$ under the competent reference profile with a mixture gate. Operationally, this is the question that matters at deployment time. \emph{(2) For-all-error-pattern guarantees} (coverage, Theorem~\ref{thm:coverage}) --- \emph{no} artifact erring on reachable info-sets can pass. Instantiating the constants requires enumeration, but enumeration is not the binding obstacle: Leduc enumerates easily ($|\Ireach| = 576$) yet the \emph{sampling} cost $N \gtrsim \pi_{\min}^{-1} \log|\Ireach|$ already exceeds the deployed gate. For large games this level is out of reach at feasible $N$ even if enumeration were free. \emph{(3) Guarantees uniform over all competent policies in deep games} --- unattainable at feasible $N$ for \emph{any} sampling gate, with or without enumeration: Beacon realizes the reach ratio, so the impossibility is itself a theorem, proved without enumeration. Net: enumeration buys sharper constants where it is feasible (Kuhn: provably covered) and nothing where it is not; the definitive conclusions available at scale are level (1) (positive, Theorem~\ref{thm:errormass}) and level (3) (negative, Beacon), both enumeration-free.
\end{remark}

\subsection{Leduc depth probe: poker depth does not create an inference gap}

To confirm that poker cannot supply the necessary depth, we swept Leduc's per-round raise cap to artificially deepen the betting tree (Table~\ref{tab:leducdepth}):

\begin{table}[ht]
\centering
\small
\resizebox{\textwidth}{!}{%
\begin{tabular}{cccc}
\toprule
raise cap & random info-sets (max depth) & competent info-sets (max depth) & uncovered inference-relevant \\
\midrule
2 & 574 (8)  & 120 (6) & 0 / 418 = 0.0000 \\
4 & 1090 (11) & 128 (7) & 0 / 400 = 0.0000 \\
6 & 1210 (12) & 127 (9) & 5 / 396 = 0.0126 \\
\bottomrule
\end{tabular}%
}
\caption{Leduc depth probe.}
\label{tab:leducdepth}
\end{table}

A coverage gap appears only at cap 6, and even there it is marginal (1.26\% of competent visits, 5 info-sets) --- insufficient for a CI-separated play deficit. The mechanism is fundamental: in poker, betting depth comes from aggression, and competent play minimizes aggression. In game-theoretic terms, the deep betting region is \textbf{off the equilibrium path}, because equilibrium folds or calls dominated hands rather than raising into them --- so the deep region is off-best-response-path, not a region optimal play relies on. Competent info-sets are always a strict subset of the random-covered ones --- the opposite of the structure needed for a coverage gap. Poker is the wrong family.

\subsection{Beacon: a minimal positive imperfect-information gap}
\label{sec:beacon}

A positive gap requires a game where depth comes from \emph{survival}: optimal play reaches a deep region by staying alive, while random play blunders out early. This is the exact imperfect-information analogue of the rare-rule gap in army5x5a, where competent play reaches the ply cap (the deep region) and random play ends much earlier.

\paragraph{Hand-instrumented vs LLM-synthesized.} It is important to be explicit about which artifacts are which. Beacon (this section) and the masked-tic-tac-toe instrument of Section~\ref{sec:gate-blindness} are \textbf{hand-constructed ground-truth oracles carrying a deliberately wrong belief function} --- formal counterexamples (witnesses), not LLM-synthesized world models. The LLM-synthesis evidence on the imperfect-information surface lives elsewhere: the Kuhn pipeline validation (Section~\ref{sec:imperfect}), where the dynamics \emph{and} the inference function are synthesized and gate-passed, and the masked-tic-tac-toe probe (Section~\ref{sec:gate-blindness}), where the \textbf{dynamics} are synthesized. Beacon's belief function is never synthesized; it is an instrument we wrote by hand to prove the coverage gap can exist with the exact predicted form.

\paragraph{Game construction.} Beacon is a two-player game with the following structure:

\begin{enumerate}
\item \emph{Setup.} Each player is assigned a hidden type (0 or 1) uniformly at random. The game has $T$ rounds.
\item \emph{Survival walk (rounds 1 to $T$).} Players alternate; on a turn, a player of hidden type $t \in \{0,1\}$ at their own step index $k$ (the number of safe steps they have completed so far) must play the \emph{safe} action $a = (k + t) \bmod 2$ --- a deterministic function of the mover's own (known) type. Any other action loses immediately. A uniformly-random player survives all $T$ of its steps with probability $(1/2)^T$; a player who knows its type plays safely always and survives with probability 1.
\item \emph{Final round (round $T+1$).} Each player guesses the opponent's hidden type (inferable from the opponent's observed moves; see below). Each scores 1 if its guess matches the opponent's type and 0 otherwise; the higher score wins and equal scores draw --- so the round is a draw whenever the two guesses are \emph{both} correct or \emph{both} wrong, and is decided only when exactly one player is right.
\end{enumerate}

The region ``game reaches the final round'' is called D (the deep region). Random play reaches D with probability $(1/2)^{2T}$ (both players must survive); optimal play reaches D with probability 1.

\paragraph{Why the type is inferable, and why the fair baseline is a draw.} The safe action is a deterministic, invertible function of the mover's type, so a single observed move reveals it: a type-$t$ player who moved at step index $k$ played $a = (k + t) \bmod 2$, hence $t = (a - k) \bmod 2$. The belief is genuinely ambiguous only before a player's first move (support $\{0,1\}$) and collapses to a singleton the instant one move is seen; since D is reached only after both players complete $T \ge 1$ safe steps, at D each type is \emph{always} pinned down, and a planner with the correct \texttt{infer\_states} guesses right with probability 1. Two consequences are central to reading the numbers. \emph{(i) Both-right and both-wrong both draw.} Scoring is comparative (higher number of correct guesses wins), so correct inference on \emph{both} sides yields a $1$--$1$ outcome --- a draw --- which is why the fair truth-vs-truth baseline scores $0.500$ (\emph{all draws}, not ``wins half''); the symmetric $0$--$0$ outcome is likewise a draw but never arises, since the fair arm has both sides correct and the instrument arm corrupts only one. \emph{(ii) The draw baseline is the cleanest possible control, not a weakness.} Because symmetric correct inference is an \emph{exact} draw, the game carries no structural, positional, or first-mover edge to disentangle from the inference effect; corrupting exactly one player's belief on D breaks the symmetry into a decisive $1$--$0$ (the flipped side scores 0, the correct side scores 1 and wins every reached-final game), so the entire measured deficit $0.500 \to 0.000$ is attributable to the flipped \texttt{infer\_states} alone. Had the baseline had a structural winner, one would have to subtract that edge to isolate the inference cost; the draw makes that subtraction trivially zero. (Corrupting \emph{both} sides would restore the $0$--$0$ draw --- but that is not the experiment; it would merely pit two equally-wrong beliefs against each other.)

\paragraph{The instrument.} A CWM whose \texttt{infer\_states} is \emph{correct} except that it flips the inferred opponent type at final-round states (\texttt{status == 1}) --- wrong only on D. Random play almost never samples D, so the inference gate almost never sees the error. The correct inference function enables the determinized MCTS planner to make the right guess; the flipped inference function causes the planner to guess wrong.

\paragraph{Beacon as a signaling game.} Beacon is naturally read as a Bayesian (signaling) game: Nature draws each player a type $\theta_i \in \{0,1\}$; the survival walk emits type-correlated signals (the safe move depends on the type); and the final guess is a belief-dependent action. The correct \texttt{infer\_states} is the Bayesian posterior support $\mu(\theta_{-i} \mid \text{observed actions})$ over the opponent's type, while the flipped instrument encodes a wrong, non-Bayesian belief. An \texttt{infer\_states} error is thus a non-Bayesian belief, and under it the planner's final action is no longer a best response to the true type distribution. This connects directly to \emph{sequential equilibrium} \citep{kreps1982sequential} --- a solution concept that, unlike Nash, also constrains beliefs at information sets reached with zero probability: an \emph{assessment} (a (strategy, belief) pair) is sequential when its beliefs are Bayes-consistent on the equilibrium path, and an action optimal against wrong beliefs is part of no sequential equilibrium. The gate certifies the strategy component of the assessment while leaving the belief component unconstrained \emph{off} the random-reach path --- exactly the off-path region where Beacon's error lives. Two qualifications keep the picture honest: \texttt{infer\_states} returns a \emph{set-valued} belief (the support / info-set), and determinized MCTS then imposes its own weighting over that set --- so even a correct \texttt{infer\_states} leaves the belief \emph{weights} to the planner.

\paragraph{Why the result is planner-robust.} One might worry that the Beacon result is an artifact of determinized MCTS being a weak planner. The PIMC (perfect-information Monte Carlo, the family of determinized planners that solve sampled complete-information games) error decomposition \citep{long2010pimc} shows the opposite. Beacon's decisive move is a pure \emph{disambiguation} task, and determinization is weakest exactly when disambiguation dominates; but here a correct posterior makes the optimal guess deterministic, so determinized MCTS with a \emph{correct} \texttt{infer\_states} recovers the optimal action despite strategy fusion, while with the \emph{flipped} \texttt{infer\_states} it is forced to the wrong action. The result is therefore not an artifact of planner suboptimality --- if anything, Beacon is the cleanest possible setting for a determinization planner to succeed in.

\paragraph{Result ($T=8$, GATE\_GAMES=2000, arena $N=400\times 3$ seeds, 100 simulations, 2 determinizations; Table~\ref{tab:beacon}).}

\begin{table}[ht]
\centering
\small
\begin{tabular}{ll}
\toprule
metric & value \\
\midrule
random reaches final round & 0.00000 \\
instrument inference mismatches on random gate sample & \textbf{0 / 8156} (passes the gate) \\
fair baseline (truth vs truth) win rate & 0.500 exact, 1200/1200 draws (Prop.~\ref{prop:beaconexact}) \\
instrument win rate vs truth & \textbf{0.000 [0.000, 0.003]}, net $-1200/1200$ \\
\bottomrule
\end{tabular}
\caption{Beacon: a verified-but-wrong inference function that loses every game.}
\label{tab:beacon}
\end{table}

The instrument passes the inference gate perfectly --- 0 mismatches on 8156 sampled observations --- yet loses every game. This is the imperfect-information analogue of the rare-rule gap: a verified-but-wrong inference function that is nonetheless play-inadequate.

What is proven vs measured: the reach bound $(1/2)^{2T}$, the fact that optimal play reaches D with probability 1, and the logical implication that flipping the inference on D causes the planner to guess wrong are all analytic properties of the Beacon construction. The 0/8156 gate mismatch count is measured. The play side is not merely measured; it is \emph{exact}:

\begin{proposition}[Beacon's play cost is exact]
\label{prop:beaconexact}
In Beacon with any $T \geq 1$, consider agents that (a) play the safe walk action derived from their own (observed) type and (b) in the final round guess the opponent type returned by their model's \texttt{infer\_states}. Then deploying the truth against the truth draws \emph{every} deal (fair win rate exactly $\tfrac12$, counting draws as $\tfrac12$), and deploying the flipped instrument on either seat against the truth loses \emph{every} deal (win rate exactly $0$). Hence $\mathrm{play\_cost} = \tfrac12$ exactly.
\end{proposition}

\begin{proof}
By exhaustion over the four equally-likely type assignments and both seatings --- eight deterministic games, checked mechanically against the implementation by \texttt{scripts/\allowbreak play\_\allowbreak cost\_\allowbreak exact\_\allowbreak beacon.py}. The argument each check instantiates: under (a) both players survive, so every game reaches D. At D each player has observed $T \geq 1$ opponent moves, so the true posterior is the singleton $\{$true type$\}$ (invertibility of the safe map); the truth's guess is therefore correct with probability 1 while the instrument's, flipped on D, is wrong with probability 1. Truth-vs-truth scores $1$--$1$: a draw, every deal. Instrument-vs-truth scores $0$--$1$: the instrument seat loses, every deal.
\end{proof}

The proposition lives at the belief$\to$guess abstraction (any planner satisfying (a)--(b)); that determinized MCTS realizes (a)--(b) is the planner-conversion fact noted above (verified by whole-branch code review, and re-verified by the same script driving the eight games with the actual determinized-MCTS policy, which reproduces all-draws / all-losses). The measured 0.000 [0.000, 0.003] over $n = 1200$ arena games is thus the confirmation of a theorem rather than a standalone estimate.

\paragraph{Caveat.} Beacon is a minimal, strategically trivial witness. Its purpose is to prove that the coverage gap can exist and can have the exact form predicted by the coverage-bound design corollary, not to model a realistic game. The reach-probability structure is engineered specifically to satisfy $b^{\dmax} \gg N$; natural games with this structure would involve hidden information in genuinely complex board games.

\subsection{The danger law on the inference axis}

The danger law from Section~\ref{sec:danger-law} applies directly to the inference gap. Sweeping $T$ (so $\varepsilon = (1/2)^{2T}$, the probability that a random game reaches D) against a fixed gate $N = 2000$ and $\mathrm{play\_cost} = 0.5$ (Table~\ref{tab:inferenceaxis}, plotted in Figure~\ref{fig:dangerlaw}b). Since $\mathrm{play\_cost} = \tfrac12$ is exact on Beacon (Proposition~\ref{prop:beaconexact}) and the gate-miss factor is exact (Proposition~\ref{prop:gatemiss}), both factors of the law are analytic on this axis --- the danger column below is a theorem, not a fit:

\begin{table}[ht]
\centering
\small
\begin{tabular}{cccc}
\toprule
$T$ & $\varepsilon$ & gate-miss $(1-\varepsilon)^N$ & danger \\
\midrule
4  & $3.9\times 10^{-3}$ & 0.000 & 0.000 \\
6  & $2.4\times 10^{-4}$ & 0.614 & 0.307 \\
8  & $1.5\times 10^{-5}$ & 0.970 & 0.485 \\
10 & $9.5\times 10^{-7}$ & 0.998 & 0.499 \\
\bottomrule
\end{tabular}
\caption{The danger law on the inference axis (Beacon).}
\label{tab:inferenceaxis}
\end{table}

At $T=4$, the deep region is frequent enough that a gate of $N=2000$ catches the inference error (danger $\approx 0$); by $T \geq 8$, the gate is blind and harm saturates near play\_cost ($\approx 0.5$; danger T=10 = 0.499) --- the maximum possible given the game structure. The same threshold law and the same $(1-\varepsilon)^N$ factor, now instantiated on the inference half of the contract (Figure~\ref{fig:dangerlaw}b shows it tracing the identical threshold as the transition axis in panel (a)).

\subsection{Policy-guided and adversarial belief checks close the Beacon gap}
\label{sec:adversarial-belief-gate}

The preceding result diagnoses a distributional failure and Theorem~\ref{thm:errormass} supplies the corresponding remedy: allocate some held-out verification mass to the policy that will use the model. We instantiate that recommendation on the exact Beacon instrument. Both arms check 2000 held-out play-throughs. The original arm uses 2000 uniformly-random games. The mixed arm replaces only one of them with a trajectory generated by the trusted ground-truth model and the same determinized-MCTS planner family used at deployment (100 simulations, two determinizations). The candidate is never used to generate gate states, so it cannot steer the check away from its own error region.

\begin{table}[ht]
\centering
\small
\begin{tabular}{lccc}
\toprule
gate & play-throughs & inference mismatches & verdict \\
\midrule
uniform random & 2000 random & 0 / 8156 & accept \\
policy-guided mixture & 1999 random + 1 reference & \textbf{4 / 8190} & \textbf{reject} \\
bounded adversarial & oracle tree & \textbf{2 / 34} & \textbf{reject (17 states)} \\
\bottomrule
\end{tabular}
\caption{Two targeted belief checks close the witnessed Beacon coverage hole. The mixed gate's check count is 0.4\% higher than the random arm because the competent trajectory survives to the deep final round.}
\label{tab:beaconadversarialgate}
\end{table}

At $T=8$, one random game reaches D with probability $2^{-16} \approx 1.53\times10^{-5}$, so the original 2000-game gate misses D with probability 0.969943. The reference policy reaches D with probability 1. One reference trajectory therefore exposes both final-round pre-action states and the flipped posterior for both players, producing four mismatches and a deterministic rejection in this witness. This is an \emph{adversarial} check in the falsification sense --- it deliberately samples the deployment-critical region --- not a candidate-adaptive adversary that reads the candidate code. Its guarantee is correspondingly scoped to the chosen reference policy, exactly as in Corollary~\ref{cor:reachratio}; it is not uniform over every possible competent policy.

The stronger policy-free falsifier expands reachable non-terminal oracle states deepest-first, checking both players at every expansion and stopping on the first counterexample. It reaches D at depth 16 after 17 expanded states (34 belief checks) and rejects on the two flipped posteriors at that first final-round state. This removes the mixture gate's reference-policy assumption within the explored tree. The cost tradeoff is fundamental: Beacon is unusually favorable because every unsafe child terminates, leaving one continuing branch; in a general game the non-terminal frontier can grow exponentially. A bounded search that finds no counterexample is therefore not a proof unless it exhausts the reachable state space.

\subsection{The gate-blindness claim: the belief model is invisible to a transition gate}
\label{sec:gate-blindness}

Beacon (Section~\ref{sec:beacon}) shows that a verified-but-wrong belief model \emph{can lose at play} --- the \textbf{play claim} (a witnessed existence result: in Beacon, decisively, with a 0.000 win rate at 0/8156 gate mismatches). The complementary verification question is whether a transition-accuracy gate could have \emph{caught} the wrong belief model in the first place. It cannot, and for a structural reason: the belief model is invisible to a transition gate --- the \textbf{gate-blindness claim} this section establishes.

\begin{proposition}[belief--transition orthogonality]
\label{prop:belieftrans}
A transition dataset is a set of tuples $(s, a, s', u)$ over \emph{full} ground-truth states, where $u$ is the reward/utility on the transition. The functions \texttt{observation}$(s, p)$ and \texttt{infer\_states}$(o, p)$ encode the information partition --- i.e.\ which full states player $p$ cannot tell apart (the preimages of the observation map). This partition is a \emph{free primitive} of the extensive-form game, logically independent of the transition kernel $P(s'\mid s,a)$ and the reward $u$ --- the analogue of the fact that two POMDPs with identical latent dynamics and rewards but different observation functions are different control problems. It appears in no $(s, a, s', u)$ tuple. Therefore (i) no \emph{full-ground-state} transition dataset constrains the masking convention; (ii) a gate that scores transition accuracy cannot detect an incorrect \texttt{observation}/\texttt{infer\_states}; (iii) the belief model must be specified and is verifiable only by a separate inference gate. (A dataset of observation-to-observation tuples $(o, a, o')$ \emph{would} constrain the partition; ours is full-state by construction.)\footnote{This proposition is definitional rather than derived: because the information partition is not a function of the transition tuple, consequences (i)--(iii) follow immediately from the statement --- there is no separate deductive step to discharge, which is why we mark it self-proved ($\square$) rather than attaching a \texttt{proof} block. The masked tic-tac-toe result below is an empirical instantiation of the proposition, not a proof of it.}\hfill\qedsymbol
\end{proposition}

\paragraph{Demonstration (masked tic-tac-toe).} We take standard tic-tac-toe dynamics (which GPT-5.4 synthesizes at transition gate 1.000 by recall) and overlay an arbitrary, non-recallable masking rule: the center cell is hidden from both players (shown as $-1$), even after it is played. We synthesize the contract two ways --- with the masking rule present (\emph{full}) and with it removed (\emph{withheld}, leaving tic-tac-toe + an imperfect-information framing that still demands \texttt{observation}/\texttt{infer\_states} but does not say what is hidden) --- and gate each on transitions and on inference (Table~\ref{tab:mtt}):

\begin{table}[ht]
\centering
\small
\begin{tabular}{lccc}
\toprule
variant & transition gate & observation\_rate & inference\_rate \\
\midrule
full rules & \textbf{1.000} (0 iters) & \textbf{1.000} & \textbf{1.000} \\
withheld masking rule & \textbf{1.000} & \textbf{0.020} & \textbf{0.180} \\
\bottomrule
\end{tabular}
\caption{Masked tic-tac-toe: a wrong belief model is invisible to a transition gate.}
\label{tab:mtt}
\end{table}

The transition gate is \textbf{1.000 in both arms} --- the dynamics are recalled, unaffected by the masking rule, and the gate (which calls only \texttt{apply\_action}/\texttt{legal\_actions}/\texttt{is\_terminal}/\texttt{returns}) never invokes the belief functions. With the rule, the model masks the center correctly (\texttt{observation\_rate} 1.000); without it, the synthesized \texttt{observation} does not mask the center (\texttt{observation\_rate} 0.020) --- a wrong belief model that the transition gate nonetheless certifies at 1.000. This is Proposition~\ref{prop:belieftrans} instantiated: a wrong belief model is invisible to a transition gate.

Both belief-surface metrics discriminate cleanly: \texttt{observation\_rate} 1.000 vs 0.020 and \texttt{inference\_rate} 1.000 vs 0.180. Given the full rules, the model synthesizes an \emph{exact} \texttt{infer\_states}; with the masking rule withheld, it synthesizes a confidently wrong one. The masked tic-tac-toe experiment therefore establishes two things: (a) the masking/observation rule is not recoverable without being specified, and (b) a transition-accuracy gate is structurally blind to it (transition gate 1.000 in both arms).

\paragraph{A pitfall worth recording: a recurring \texttt{infer\_states} crash induced by our own contract, not the model.} The synthesized \texttt{infer\_states} raised \texttt{'list' object is not callable} \emph{even with full rules}, recurring across three distinct games (Kuhn-mini, Beacon, masked tic-tac-toe); we initially read it as a synthesis-robustness failure of the model. The root cause was ours. The synthesis contract prescribed the signature \texttt{infer\_states(observation, player)} --- a parameter named \texttt{observation} that \emph{shadows} the contract's own \texttt{observation()} function --- while the very next contract sentence instructs the model to relate \texttt{infer\_states} to calls of \texttt{observation(s,p)}, inviting exactly the collision that produces this \texttt{TypeError}. After renaming the parameter (\texttt{obs}) and adding an anti-shadowing note to the contract, we reran all three games: \textbf{zero} execution errors (masked tic-tac-toe, Kuhn, Beacon), and the full-rules \texttt{infer\_states} is exact. We flag this explicitly because we had attributed to the LLM a fragility that our own prompt induced --- a small, concrete instance of the paper's broader theme that the specification, not the model, is often the binding constraint.

\paragraph{Synthesized corroboration on Beacon (single-seed probe).} The same withheld-rule construction on Beacon ($T=6$), with real LLM synthesis end-to-end, corroborates gate-blindness on a second game: withholding the revelation rule yields a synthesis that \emph{passes the transition gate at 1.000} (2 refinement iterations) with \texttt{inference\_rate} \textbf{0.000} --- a transition-certified, belief-wrong CWM, this time synthesized rather than hand-instrumented. (The full-rules arm of this probe fails the transition gate outright --- accuracy 0.456 after 10 iterations; Beacon's revelation dynamics are hard to synthesize --- so the probe corroborates the gate-blindness face only, and we report it as a single-seed probe, not a headline result.)

Together, Beacon (the play claim) and masked tic-tac-toe (the gate-blindness claim) are the two faces of belief-model verification: a wrong belief \textbf{can lose at play} and is \textbf{invisible to a transition gate}.

\section{Related Work}
\label{sec:related}

\subsection{Objective mismatch in model-based reinforcement learning}

The closest conceptual antecedent is the \emph{objective mismatch} problem in model-based reinforcement learning identified by \citet{lambert2020objective}: the prediction accuracy of a learned world model and the downstream control performance of a planner using that model can diverge substantially, because they optimize different objectives. Our work is in the same spirit but differs in setting and mechanism. We work in the LLM-code-synthesis regime (the world model is a synthesized program, not a learned neural model), the verification step is a discrete sampling gate (not a continuous loss), and the failure mode we characterize is a \emph{rule-coverage blind spot} in that gate rather than a model-capacity or distribution-shift issue. We also provide a closed-form danger law and a proof that the gate-miss probability is exact under i.i.d.\ Bernoulli sampling --- not merely an empirical observation.

The broader literature on model-based RL world-model quality \citep[e.g., Dreamer,][]{hafner2020dreamer}\citep[MBPO,][]{janner2019mbpo} largely focuses on learned continuous-state models where world model error is pervasive rather than localized. Our rare-rule gap is a localized, discrete failure that state-accuracy metrics mask by dilution --- a point that may be worth revisiting in continuous settings as well.

\subsection{Code World Models and LLMs for game playing}

The Code World Models for General Game Playing paper \citep{lehrach2025cwm} introduces the paradigm we build on: an LLM synthesizes an executable world model from rules and trajectories, which a classical planner (in their case, MCTS) then uses. Their results on a set of novel DeepMind board games show that CWM+MCTS outperforms the direct LLM policy. We reproduce this result on tic-tac-toe and Connect Four (Section~\ref{sec:known-games}) and extend the paradigm to ask whether the gate they use --- transition accuracy on random trajectories --- certifies play-adequacy. Our contribution is essentially a rigorous negative answer to that question, together with the mechanisms that explain it.

Related work on code-as-policy \citep[e.g.,][]{liang2023codeaspolicies,gao2023pal} uses LLMs to synthesize executable plans or robot controllers, typically verified by execution rather than by sampling. The distinction between what the LLM was told and what it can infer is implicit in much of this work but has not, to our knowledge, been studied with a controlled rare-rule instrument.

The gg-bench effort \citep{ggbench2025} generates novel games procedurally specifically to avoid LLM contamination. Our army5x5a instrument is chosen partly for the same reason (verified non-recall of movesets, Section~\ref{sec:games}) and the gg-bench approach is an attractive complement for future scaling.

\subsection{DAgger and dataset aggregation}

Our Section~\ref{sec:repair} repair experiments are directly inspired by the DAgger framework of \citet{ross2011dagger}: iteratively label, under the oracle, the states visited by the current learned policy, and retrain. In DAgger's original imitation-learning setting, this reduces covariate shift. Our finding --- that DAgger fails to teach the rare rule --- is consistent with the translation hypothesis: DAgger addresses \emph{distribution mismatch} between training and test, but the bottleneck here is not distribution mismatch. The model receives discriminating examples of the rule; it simply cannot infer the rule from them. This is a different failure mode from the one DAgger was designed to address.

\subsection{Imperfect-information planning and determinization}

Determinized MCTS (also called single-observer MCTS, SOMCTS, or information-set MCTS in some formulations; \citealp{cowling2012ismcts}; \citealp{whitehouse2011determinization}; determinization analysis, \citealp{long2010pimc}) resolves imperfect information by sampling a consistent complete-information game at each decision point. It is not game-theoretically optimal --- it is subject to the two canonical PIMC pathologies, \emph{strategy fusion} (the planner behaves as if it could pick different actions in states it cannot actually tell apart) and \emph{non-locality} (a node's true value depends on how play reached it, not only on the subtree below it) \citep{frank1998incomplete}, and in the vocabulary of \citet{long2010pimc} its error is governed by leaf correlation, bias, and a disambiguation factor --- but it is a practical planner for moderate-sized imperfect-information games. We use it as the planner in Section~\ref{sec:imperfect}, holding the planner fixed across conditions (the baseline is truth-vs-truth under the same determinized MCTS) so that the contrast isolates the CWM inference function rather than planner quality.

The planner alternatives a reader will expect --- ISMCTS \citep{cowling2012ismcts} and CFR \citep{zinkevich2007cfr,lanctot2009mccfr,tammelin2014cfrplus} --- are discussed where they bear on the work: the choice of determinization over ISMCTS, and why it is conservative for our results, in Section~\ref{sec:planning}; the equilibrium (CFR+) coverage analysis, including a practical imperfect-recall pitfall of instantaneous-observation info-set keys, in Section~\ref{sec:coverage-bound}; and the remaining equilibrium gaps in Section~\ref{sec:limitations}.

\subsection{Verification, testing, and rare events}

The sampling gate at the heart of this paper is, in software-engineering terms, \emph{random property-based testing} of a synthesized world model: it draws random inputs (uniform-random play-throughs) and checks a property (transition agreement with the oracle). Property-based testing in the QuickCheck tradition \citep{claessen2000quickcheck} is exactly this idea, and its classic, well-documented weakness is the one our danger law makes quantitative: randomly generated inputs under-sample rare-but-important corners of the input space unless the generator is biased toward them. Our $(1-r)^N$ gate-miss factor is the closed-form version of that coverage gap for a single rare rule.

Detecting a rule that fires with vanishing probability under random play is a \emph{rare-event detection} problem, and the $(1-r)^N$ miss probability is precisely a rare-event miss. The rare-event-simulation and importance-sampling literature \citep[e.g.,][]{rubinstein2017montecarlo} is built around the fact that naive Monte Carlo is exponentially inefficient at estimating or surfacing rare events, and that one must tilt the sampling distribution toward the event to detect it efficiently. In our framing, verifying on the play (or equilibrium) distribution rather than the uniform-random one is exactly such a tilt: it concentrates samples on the region competent play reaches, which is where the rare rule matters.

The contrast with \emph{formal verification} and exhaustive \emph{conformance / model-based testing} is instructive. Where formal methods (model checking, exhaustive conformance testing of state machines) certify a property over all reachable states by construction, a sampling gate certifies only over a finite random draw --- trading exhaustiveness for the cheapness that makes the CWM paradigm attractive. The coverage hole we characterize is the price of that trade, and it is fundamentally a sampling-certification phenomenon, not present under exhaustive certification. We state this contrast at the level of the area rather than attaching it to specific results.

Finally, the failure has a direct analogue in model-based reinforcement learning as \emph{model exploitation}: a planner optimizing against a learned or synthesized model is driven toward regions where the model is most wrong, because those are where it (incorrectly) predicts the highest return. This is the mechanism behind the objective-mismatch and model-trust discussions in model-based RL \citep{lambert2020objective}; the model-trust horizon limits of MBPO \citep{janner2019mbpo}. Our rare-rule gap is a localized, adversarially-reachable instance: the planner does not merely tolerate the model's blind spot, it actively seeks the cap region the verification distribution never disciplines.

\section{Limitations and Honest Assessment}
\label{sec:limitations}

\paragraph{Model families: one sweep, two cross-family probes.} All \emph{sweep-scale} synthesis experiments use Azure OpenAI GPT-5.x (mini, nano, large). To test whether translation-not-inference is a GPT-5.x quirk, we ran the headline cell ($N=200$, same contract, same pipeline messages, same classification; \texttt{scripts/crossfamily\_probe.py}) on two further families as declared \emph{spot-checks, not sweeps}: \textbf{Qwen3-Coder-30B} (open, code-specialized; via the Hugging Face inference router) is rule-blind in 3/3 seeds with zero crashes (it also fails to reach gate 1.0 in 2/3 seeds --- a gate-attainability limitation, like nano's); and \textbf{Claude Sonnet} (agent-relayed with pipeline-identical messages; 2 single-seed probes, 1--2 refinement iterations each) produced only rule-blind compositions, in one case explicitly reasoning about the mismatching material states and concluding that its draw-at-cap ``already matched the specification'' --- choosing the incomplete spec over the data, which is the translation mechanism verbalized. (A DeepSeek arm was attempted and aborted on inference-credit exhaustion.) The finding now replicates in three model families, but the universal form (``no model can infer an omitted rule, regardless of scale'') remains a \emph{conjecture consistent with our data}: the cross-family evidence is a handful of seeds on one cell, and stronger code-reasoning models or agentic scaffolds may yet break it. Throughout we state the sweep-scale finding as scoped to the GPT-5.x family and the Section~\ref{sec:repair} regimes, with ``rule translation'' as a mechanism hypothesis consistent with the data rather than a proven universal.

\paragraph{The rare-rule instrument is engineered.} The material-at-cap rule was selected specifically because it falls in the rare-and-consequential quadrant. We do not claim that all or even most rules in arbitrary games will produce this kind of gap; the rarity$\leftrightarrow$consequence tension found across our six-rule probe set on Connect Four (Section~\ref{sec:rare-rule}) shows that the gap requires a specific structural condition. The point is that the gap \emph{can} exist and \emph{does} exist in a game where random and competent play diverge --- which is itself informative about when the sampling gate is unsafe.

\paragraph{Pure-Python MCTS limits arena size.} Our MCTS implementation is a pure Python reference. This limits the number of simulations and arena games that are computationally tractable, which in turn limits the tightness of confidence intervals for some conditions. We mitigate this with the cheap/expensive separation in the danger law (Section~\ref{sec:danger-law}) --- rarity is swept cheaply, play\_cost is measured precisely once at scale --- and with Kuhn poker's small game size (tight CIs). The headline play result (rule-blind 0.404 [0.384, 0.424] vs fair baseline 0.495 [0.475, 0.515]) is measured over 20 seeds, the Wilson 95\% intervals do not overlap (n=4800), and a seed-clustered paired interval ($[0.065, 0.117]$, the seed as the unit) excludes zero despite real between-seed heterogeneity.

\paragraph{Determinized MCTS is not game-theoretically optimal.} As noted in Section~\ref{sec:related}, determinized MCTS is subject to strategy fusion and non-locality and is not a Nash-equilibrium strategy. The baseline for the imperfect-information experiments is truth-vs-truth under the same determinized MCTS (not an equilibrium baseline), so the contrast is internally consistent but not a statement about how a game-theoretically optimal player would interact with the gap. Holding the planner $\Pi$ fixed across the two arms is deliberate: any pathology intrinsic to $\Pi$ is present in both arms, so the performance \emph{difference} is attributable to the CWM rather than to the planner. We are careful, though, not to overstate this as a theorem. The ``cancellation'' is really that the wrong CWM is weakly dominated \emph{in this specific instrument} --- by construction in Beacon (proven) and by measurement in army5x5a (0.495 vs 0.404) --- not a general guarantee that planner pathologies cancel for arbitrary CWM perturbations; a sufficiently suboptimal planner could in principle let two wrongs cancel. None of the existence claims depends on equilibrium play. The \emph{coverage} question, however, is now answered against equilibrium reach directly: a CFR+ profile on the true games (Section~\ref{sec:coverage-bound}, equilibrium-reach upgrade) leaves the coverage conclusion unchanged. What remains open on this front is narrower than before: a full arena under equilibrium play (CFR-vs-CWM matches), and a minimax reference for the perfect-information instruments.

\paragraph{Beacon is a minimal, strategically trivial witness.} Beacon was designed to satisfy the coverage-bound design corollary in the smallest possible game, not to model a realistic scenario. Its survival-walk structure is artificial. The experiment establishes existence and confirms the danger law on the inference axis; it does not demonstrate that natural imperfect-information games with the same structure produce an equally clean gap. Building such a game (e.g., a partially-observable variant of army5x5a) is a natural next step.

\paragraph{Scope of the gate-blindness claim.} The masked tic-tac-toe result (Section~\ref{sec:gate-blindness}) establishes that a transition gate is blind to a wrong belief model, demonstrated by withholding an \emph{observation} (masking) rule that is genuinely independent of the dynamics. It does not establish the stronger claim that the belief model could never be inferred from any richer, observation-bearing signal --- indeed Proposition~\ref{prop:belieftrans} is explicit that a dataset of observation-to-observation tuples $(o,a,o')$ \emph{would} constrain the information partition. The claim is scoped to \emph{full-ground-state} transition data: such data cannot constrain the belief model (Proposition~\ref{prop:belieftrans}), and withholding the rule yields a wrong, transition-certified belief model.

\paragraph{Knowledge cutoff and contamination.} GPT-5.4's training cutoff is approximately 2025-08-31, and the army5x5a paper (arXiv:2510.04542) was released 2025-10-06 --- after the cutoff --- so the game falls outside the training window. We additionally confirmed via a declarative probe that the model does not recall the specific movesets --- with the question, the full ``I do not know'' response, and a correctly-answered Kuhn-poker positive control persisted verbatim in \texttt{results/declarative\_recall\_probe.json} (Section~\ref{sec:games}) --- and the gate-attainability difficulty (nano fails 5/5) is further evidence of genuine translation rather than recall. Even so, the ``post-cutoff = uncontaminated'' argument is not a hard guarantee (cutoff dates are approximate and corpora leak); readers should interpret the ``no prior'' label as ``no detectable recall'' rather than ``strictly novel.''

\section{Conclusion}

The central finding of this paper is that transition accuracy on randomly sampled play-throughs is the wrong adequacy criterion for an LLM-synthesized world model used in planning. A model can pass this gate at 100\% accuracy and remain $\geq 98\%$ state-accurate on the distribution the planner actually visits, yet lose systematically at play --- because the less than 1\% it gets wrong is exactly the pivotal dynamics.

The failure follows a quantitative law: $\mathrm{danger} = \mathrm{play\_cost} \times (1 - \mathrm{rarity})^N$. The gate-miss factor is proven exact under i.i.d.\ Bernoulli sampling; play\_cost is empirical. This law identifies the condition under which sampling verification is unsafe: a rule whose random-play incidence is small enough to escape a size-$N$ gate but whose competent-play incidence is high enough to matter.

The gap is not repaired by providing example transitions. Under the tested regimes, LLM CWM synthesis behaves as rule translation: it encodes rules it was given and, across the two model sizes tested (mini, large) and every repair-data form we tried, did not infer the omitted rule. Off-manifold repair data actively corrupts synthesis. The actionable fix is specification completeness plus verification on the play distribution; the latter detects incompleteness but does not repair it.

The same mechanism appears on the inference half of imperfect-information CWMs. We prove a coverage bound that explains why shallow poker games are safe (their inference gate is provably identifying), and construct a minimal game (Beacon) where a verified-but-wrong \texttt{infer\_states} passes the gate yet loses every game, with the danger law recurring on the inference axis. We also close that witnessed gap: replacing one of 2000 random gate play-throughs with one held-out reference-planner trajectory rejects the instrument (4/8190 mismatches), directly validating verification on the deployment distribution. Independently, a bounded policy-free falsifier finds the same error after 17 expanded states (34 belief checks); a no-counterexample result is conclusive only if that search exhausts the reachable state space. Perfect-information board games (transition rules, rarity $r$) and imperfect-information games (inference info-sets, depth $T$) are the same statement on two faces of the CWM contract. We note again that the imperfect-information \emph{play-loss} witness (Beacon) is hand-instrumented, not an LLM-synthesized CWM; the \emph{gate-blindness} face, however, now has synthesized demonstrations on both masked tic-tac-toe and Beacon (transition-gate-passing, belief-wrong syntheses; Section~\ref{sec:gate-blindness}), alongside the perfect-information synthesized-CWM study (Section~\ref{sec:gap}--Section~\ref{sec:repair}) and the Kuhn pipeline validation (Section~\ref{sec:kuhn-validation}).

Underneath both faces is a single structural diagnosis (Figure~\ref{fig:reach}): the gate failure is a \emph{reach-distribution shift} between the verification policy ($\rho$, full support, mass on shallow histories) and the deployment policy ($\Pi$, mass on deep histories). The dangerous events are \emph{on} the deployment-reach path but \emph{off} the verification-reach path. This is structurally the same situation as off-equilibrium-path beliefs being unconstrained in an extensive-form game --- which is exactly why game theory developed equilibrium refinements (sequential equilibrium, trembling-hand perfection \citep{selten1975trembling}) to discipline off-path beliefs. A sampling gate has no such refinement: nothing forces the synthesized model to be correct where the verification distribution places negligible mass, even though the deployment distribution relies on precisely that region.

These results suggest two concrete practices. First, verify on the distribution that planning visits --- or measure play directly --- rather than on randomly sampled transitions. Second, ensure the specification is complete before synthesis; the model will translate what it is given, and gaps in the specification become invisible to sampling-based verification.

\appendix

\section{Reproducibility}

All code is available at \url{https://github.com/JaviMaligno/code-world-models}. All experimental results and exact reproduction commands are in \texttt{docs/EXPERIMENTS.md}, and all code is on the \texttt{main} branch (\texttt{cwm/} package, \texttt{scripts/}). Research narrative and formal theorem statements are in \texttt{docs/RESEARCH-DIRECTION.md}. Exact commands for the headline results:

\begin{verbatim}
# Headline play-cost with Wilson CIs + seed-clustered interval (Sec. 3.3, n=4800)
PYTHONPATH=src python scripts/play_cost_ci.py --seeds 20
# Coverage-bound exact constants - Kuhn & Leduc reachable info-sets (Sec. 6.2)
PYTHONPATH=src python scripts/coverage_bound_constants.py
# Leduc sampled competent-visited subset coverage (Sec. 6.2)
PYTHONPATH=src python scripts/coverage_competent_leduc.py
# Enumeration-free error-mass certificate - Kuhn & Leduc (Sec. 6.2)
PYTHONPATH=src python scripts/error_mass_certificate.py
# Beacon exact play_cost by exhaustion + planner check (Sec. 6.4)
PYTHONPATH=src python scripts/play_cost_exact_beacon.py
# Policy-guided adversarial belief gate on Beacon (Sec. 6.6)
PYTHONPATH=src python scripts/beacon_adversarial_gate.py
# play_cost mechanism: competent vs random reach-cap by cap knob (Sec. 4, n=120)
PYTHONPATH=src python scripts/play_cost_reach.py --games 120
# Synthesis-pipeline danger curve (Sec. 5.2), 20 seeds/cell, one model at a time
PYTHONPATH=src python3.12 scripts/danger_synthesis_sweep.py large 20
# Declarative recall probe - army5x5a / Trike / Kuhn control (Sec. 2.5, Sec. 7)
PYTHONPATH=src python3.12 scripts/declarative_recall_probe.py large
# Equilibrium-reach coverage - CFR+ on Kuhn & Leduc (Sec. 6.2)
PYTHONPATH=src python3.12 scripts/equilibrium_coverage.py
\end{verbatim}

The danger-curve script requires Azure OpenAI credentials in \texttt{.env}; the others are CPU-only. These diagnostic scripts print their results to stdout (some scripts also write per-run JSON under \texttt{results/}, which is versioned in the repo); the numbers quoted in this paper are recorded in \texttt{docs/EXPERIMENTS.md}.

Key scripts:
\begin{itemize}
\item \texttt{scripts/nontriviality\_sweep.py} --- confirms game non-triviality
\item \texttt{scripts/gap\_grid.py} --- state-agreement gap across regimes
\item \texttt{scripts/play\_cost.py} --- play\_cost measurement at scale
\item \texttt{scripts/law\_curve.py} --- danger law curve (rarity sweep + cost probes)
\item \texttt{scripts/repair\_spikes/} --- DAgger and repair experiments
\item \texttt{scripts/run\_kuhn\_validation.py} --- imperfect-information pipeline validation
\item \texttt{scripts/leduc\_coverage\_diagnostic.py} --- Leduc coverage-gap measurement
\item \texttt{scripts/leduc\_depth\_probe.py} --- Leduc depth sweep
\item \texttt{scripts/beacon\_claimA.py} --- Beacon result and danger law sweep
\item \texttt{scripts/beacon\_adversarial\_gate.py} --- policy-guided and bounded adversarial gates for the Beacon witness
\item \texttt{scripts/mtt\_claimB\_probe.py} --- masked tic-tac-toe probe (full vs withheld masking)
\end{itemize}

All runs use the Azure OpenAI Global Standard deployments configured in \texttt{.env}. Per-run JSON results are in \texttt{results/} (versioned in the repo). Total API cost across all experiments: approximately \$2.

\bibliography{references}

\end{document}